\documentclass{article}

\usepackage[final,nonatbib]{neurips_2023}

\usepackage[T1]{fontenc}    
\usepackage{hyperref}       
\usepackage{url}            
\usepackage{booktabs}       
\usepackage{amsfonts}       
\usepackage{nicefrac}       
\usepackage{microtype}      
\usepackage{xcolor}         
\usepackage{subcaption}
\usepackage{tabularray}
\usepackage{multirow}
\usepackage{colortbl}

\usepackage{booktabs}
\usepackage{algorithm}
\usepackage{algorithmic}

\usepackage{bm}

\usepackage{amsthm,amsmath,amssymb}
\usepackage{graphicx}
\DeclareMathOperator*{\argmax}{argmax}
\DeclareMathOperator*{\argmin}{argmin}
\usepackage{color}

\makeatletter
\newcommand\figcaption{\def\@captype{figure}\caption}
\newcommand\tabcaption{\def\@captype{table}\caption}
\makeatother

\usepackage{tikz}
\usetikzlibrary{backgrounds}
\usetikzlibrary{arrows,shapes}
\usetikzlibrary{tikzmark}
\usetikzlibrary{calc}

\usepackage{macros}

\title{CSOT: Curriculum and Structure-Aware Optimal Transport for Learning with Noisy Labels
}

\author{%
  Wanxing Chang$^1$\hspace{1em}
  Ye Shi$^{1,2}$
  \hspace{1em}
  Jingya Wang$^{1,2}$\thanks{Corresponding author.}\\\vspace{1pt}\\
  $^1$ShanghaiTech University
  \hspace{1em}\\
  $^2$Shanghai Engineering Research Center of Intelligent Vision and Imaging \\
  \\ \texttt{\{changwx,shiye,wangjingya\}@shanghaitech.edu.cn} \\
}

\begin{document}

\maketitle

\begin{abstract}
Learning with noisy labels (LNL) poses a significant challenge in training a well-generalized model while avoiding overfitting to corrupted labels. Recent advances have achieved impressive performance by identifying clean labels and correcting corrupted labels for training. However, the current approaches rely heavily on the model’s predictions and evaluate each sample independently without considering either the global or local structure of the sample distribution.
These limitations typically result in a suboptimal solution for the identification and correction processes, which eventually leads to models overfitting to incorrect labels. In this paper, we propose a novel optimal transport (OT) formulation, called Curriculum and Structure-aware Optimal Transport (CSOT). CSOT concurrently considers the inter- and intra-distribution structure of the samples to construct a robust denoising and relabeling allocator.
During the training process, the allocator incrementally assigns reliable labels to a fraction of the samples with the highest confidence. These labels have both global discriminability and local coherence. Notably, CSOT is a new OT formulation with a nonconvex objective function and curriculum constraints, so it is not directly compatible with classical OT solvers. Here, we develop a lightspeed computational method that involves a scaling iteration within a generalized conditional gradient framework to solve CSOT efficiently.
Extensive experiments demonstrate the superiority of our method over the current state-of-the-arts in LNL.
Code is available at \url{https://github.com/changwxx/CSOT-for-LNL}.

\end{abstract}

\section{Introduction}

Deep neural networks (DNNs) have significantly boosted performance in various computer vision tasks, including image classification \cite{he2016deep}, object detection \cite{ren2015faster}, and semantic segmentation \cite{he2017mask}. 
However, the remarkable performance of deep learning algorithms heavily relies on large-scale high-quality human annotations, which are extremely expensive and time-consuming to obtain. 
Alternatively, mining large-scale labeled data based on a web search and user tags \cite{webvision, jiang2020beyond} can provide a cost-effective way to collect labels, but this approach inevitably introduces noisy labels.
Since DNNs can so easily overfit to noisy labels \cite{arpit2017closer, zhang2021understanding}, such label noise can significantly degrade performance, giving rise to a challenging task: learning with noisy labels (LNL) \cite{ELR, Decoupling, DivideMix}.

Numerous strategies have been proposed to mitigate the negative impact of noisy labels, including loss correction based on transition matrix estimation \cite{hendrycks2018using}, re-weighting \cite{ren2018learning}, label correction \cite{PENCIL} and sample selection \cite{Decoupling}.
Recent advances have achieved impressive performance by identifying clean labels and correcting corrupted labels for training.
However, current approaches rely heavily on the model’s predictions to identify or correct labels even if the model is not yet sufficiently trained.
Moreover, these approaches often evaluate each sample independently, disregarding the global or local structure of the sample distribution.
Hence, the identification and correction process results in a suboptimal solution which eventually leads to a model overfitting to incorrect labels.

In light of the limitations of distribution modeling, optimal transport (OT) offers a promising solution by optimizing the global distribution matching problem that searches for an efficient transport plan from one distribution to another.
To date, OT has been applied in various machine learning tasks \cite{SwAV,zheng2021group,guo2022learning}.
In particular, OT-based pseudo-labeling \cite{SwAV,OTCleaner} attempts to map samples to class centroids, while considering the \textit{inter-distribution} matching of samples and classes. 
However, such an approach could also produce assignments that overlook the inherent coherence structure of the sample distribution, \ie \textit{intra-distribution} coherence.
More specifically, the cost matrix in OT relies on pairwise metrics, so two nearby samples could be mapped to two far-away class centroids (Fig. \ref{fig:overview}).

In this paper, to enhance intra-distribution coherence, we propose a new OT formulation for denoising and relabeling, called Structure-aware Optimal Transport (SOT). 
This formulation fully considers the intra-distribution structure of the samples and produces robust assignments with both \textit{global discriminability} and \textit{local coherence}.
Technically speaking, we introduce local coherent regularized terms to encourage both prediction- and label-level local consistency in the assignments.
Furthermore, to avoid generating incorrect labels in the early stages of training or cases with high noise ratios, we devise Curriculum and Structure-aware Optimal Transport (CSOT) based on SOT.
CSOT constructs a robust denoising and relabeling allocator by relaxing one of the equality constraints to allow only a fraction of the samples with the highest confidence to be selected.
These samples are then assigned with reliable pseudo labels.
The allocator progressively selects and relabels batches of high-confidence samples based on an increasing budget factor that controls the number of selected samples. 
Notably, CSOT is a new OT formulation with a nonconvex objective function and curriculum constraints, so it is significantly different from the classical OT formulations.
Hence, to solve CSOT efficiently, we developed a lightspeed computational method that involves a scaling iteration within a generalized conditional gradient framework \cite{GeneralizedCG}.

Our contribution can be summarized as follows:
1) We tackle the denoising and relabeling problem in 
LNL from a new perspective, i.e. simultaneously considering the \textit{inter-} and \textit{intra-distribution} structure for generating superior pseudo labels using optimal transport.
2) To fully consider the intrinsic coherence structure of sample distribution, we propose a novel optimal transport formulation, namely Curriculum and Structure-aware Optimal Transport (CSOT), which constructs a robust denoising and relabeling allocator that mitigates error accumulation.
This allocator selects a fraction of high-confidence samples, which are then assigned  reliable labels with both \textit{global discriminability} and \textit{local coherence}. 
3) We further develop a lightspeed computational method that involves a scaling iteration within a generalized conditional gradient framework to efficiently solve CSOT. 
4) Extensive experiments demonstrate the superiority of our method over state-of-the-art methods in LNL.

\section{Related Work}
\paragraph{Learning with noisy labels.}
LNL is a well-studied field with numerous strategies having been proposed to solve this challenging problem, such as robust loss design \cite{zhang2018generalized, wang2019symmetric}, loss correction \cite{hendrycks2018using, F-correction}, loss re-weighting \cite{ren2018learning,zhang2021dualgraph} and sample selection \cite{Decoupling, Co-teaching, karim2022unicon}.
Currently, the methods that are delivering superior performance mainly involve learning from both selected clean labels and relabeled corrupted labels \cite{DivideMix, NCE}. 
The mainstream approaches for identifying clean labels typically rely on the small-loss criterion \cite{Co-teaching, Co-teaching_plus, JoCoR, INCV}.
These methods often model per-sample loss distributions using a Beta Mixture Model \cite{ma2011bayesian} or a Gaussian Mixture Model \cite{permuter2006study}, treating samples with smaller loss as clean ones \cite{LossModelling, JoCoR, DivideMix}.
The label correction methods, such as PENCIL \cite{PENCIL}, Selfie \cite{Selfie}, ELR \cite{ELR}, and DivideMix \cite{DivideMix}, typically adopt a pseudo-labeling strategy that leverages the DNN’s predictions to correct the labels.
However, these approaches evaluate each sample independently without considering the correlations among samples, which leads to a suboptimal identification and correction solution.
To this end, some work \cite{MOIT, NCE} attempt to leverage $k$-nearest neighbor predictions \cite{bahri2020deep} for clean identification and label correction.
Besides, to further select and correct noisy labels robustly, OT Cleaner \cite{OTCleaner}, as well as concurrent OT-Filter \cite{OT-filter}, designed to consider the global sample distribution by formulating pseudo-labeling as an optimal transport problem.
In this paper, we propose CSOT to construct a robust denoising and relabeling allocator that simultaneously considers both the global and local structure of sample distribution so as to generate better pseudo labels.

\paragraph{Optimal transport-based pseudo-labeling.}
OT is a constrained optimization problem that aims to find the optimal coupling matrix to map one probability distribution to another while minimizing the total cost \cite{kantorovich2006translocation}.
OT has been formulated as a pseudo-labeling (PL) technique for a range of machine learning tasks, including class-imbalanced learning \cite{SaR, guo2022learning, SoLar}, 
semi-supervised learning \cite{SLA, ConfidentSLA, SaR}, 
clustering \cite{Self-labelling, SwAV, UNO}, 
domain adaptation \cite{zheng2021group, UniOT},
label refinery \cite{zheng2021group, SoLar, OTCleaner, OT-filter}, and others.
Unlike prediction-based PL \cite{fixmatch}, OT-based PL optimizes the mapping samples to class centroids, while considering the global structure of the sample distribution in terms of marginal constraints instead of per-sample predictions.
For example, Self-labelling \cite{Self-labelling} and SwAV \cite{SwAV}, which are designed for self-supervised learning, both seek an optimal equal-partition clustering to avoid the model’s collapse.
In addition, because OT-based PL considers marginal constraints, it can also consider class distribution to solve class-imbalance problems \cite{SaR, guo2022learning, SoLar}.
However, these approaches only consider the inter-distribution matching of samples and classes but do not consider the intra-distribution coherence structure of samples.
By contrast, our proposed CSOT considers both the inter- and intra-distribution structure and generates superior pseudo labels for noise-robust learning.

\paragraph{Curriculum learning.} 
Curriculum learning (CL) attempts to gradually increase the difficulty of the training samples, allowing the model to learn progressively from easier concepts to more complex ones \cite{khan2011humans}.
CL has been applied to various machine learning tasks, including image classification \cite{jiang2015self, zhou2021curriculum}, and reinforcement learning \cite{narvekar2020curriculum, ao2021co}.
Recently, the combination of curriculum learning and pseudo-labeling has become popular in semi-supervised learning.
These methods mainly focus on dynamic confident thresholding \cite{freematch, guo2022class, Dash} instead of adopting a fixed threshold \cite{fixmatch}.
Flexmatch \cite{flexmatch} designs class-wise thresholds and lowers the thresholds for classes that are more difficult to learn. 
Different from dynamic thresholding approaches, SLA \cite{SLA} only assigns pseudo labels to easy samples gradually based on an OT-like problem.
In the context of LNL, CurriculumNet \cite{Curriculumnet} designs a curriculum by ranking the complexity of the data using its distribution density in a feature space.
Alternatively, RoCL \cite{RoCL} selects easier samples considering both the dynamics of the per-sample loss and the output consistency.
Our proposed CSOT constructs a robust denoising and relabeling allocator that gradually assigns high-quality labels to a fraction of the samples with the highest confidence. 
This encourages both global discriminability and local coherence in assignments.

\section{Preliminaries}
\paragraph{Optimal transport.}
Here we briefly recap the well-known formulation of OT. Given two probability simplex vectors $\bmAlpha$ and $\bmBeta$ indicating two distributions, as well as a cost matrix $\mathbfC \in \mathbb{R}^{|\bmAlpha| \times |\bmBeta|}$, where $|\bmAlpha|$ denotes the dimension of $\bmAlpha$,
OT aims to seek the optimal coupling matrix $\mathbfQ$ by minimizing the following objective
\begin{equation}
    \label{Equation:OT}
        \min_{\mathbfQ\in\OTpolytope}
            \left<\mathbfC, \mathbfQ\right>,
\end{equation}
where $\left<\cdot,\cdot\right>$ denote Frobenius dot-product.
The coupling matrix $\mathbfQ$ satisfies the polytope
$\OTpolytope=\left\{\mathbfQ\in\mathbb{R}_{+}^{|\bmAlpha| \times |\bmBeta|}
|\mathbfQ\ones_{|\bmBeta|}=\bmAlpha,\;
\mathbfQ^\top\ones_{|\bmAlpha|}=\bmBeta
\right\}$, where $\bmAlpha$ and $\bmBeta$ are essentially marginal probability vectors.
Intuitively speaking, these two marginal probability vectors can be interpreted as coupling budgets, which control the mapping intensity of each row and column in $\mathbfQ$.

\paragraph{Pseudo-labeling based on optimal transport.}
Let $\mathbfP\in\mathbb{R}_{+}^{B \times C}$ denote classifier softmax predictions, where $B$ is the batch size of samples, and $C$ is the number of classes.
The OT-based PL considers mapping samples to class centroids and the cost matrix $\mathbfC$ can be formulated as $-\log\mathbfP$ \cite{SLA, SoLar}. 
We can rewrite the objective for OT-based PL based on Problem (\ref{Equation:OT}) as follows
\begin{equation}
    \label{eq:OTPL}
    \min_{\mathbfQ\in\mathbfPi(\frac{1}{B}\ones_{B},
        \frac{1}{C}\ones_{C})}
                \left<-\log\mathbfP, \mathbfQ\right>,
\end{equation}
where $\ones_{d}$ indicates a $d$-dimensional vector of ones.
The pseudo-labeling matrix can be obtained by normalization: $B\mathbfQ$.
Unlike prediction-based PL \cite{fixmatch} which evaluates each sample independently, OT-based PL considers inter-distribution matching of samples and classes, as well as the global structure of sample distribution, thanks to the equality constraints. 

\paragraph{Sinkhorn algorithm for classical optimal transport problem.} 
Directly optimizing the exact OT problem would be time-consuming, and an entropic regularization term is introduced \cite{sinkhorn}: 
$\min_{\mathbfQ\in
\mathbfPi(\bmAlpha,\bmBeta)}
    \left<\mathbfC, \mathbfQ\right>
+\varepsilon\left<\mathbfQ, \log\mathbfQ\right>$
, where $\varepsilon>0$. The entropic regularization term enables OT to be approximated efficiently by the Sinkhorn algorithm \cite{sinkhorn}, which involves matrix scaling iterations executed efficiently by matrix multiplication on GPU.

\begin{figure*}
\begin{center}
\includegraphics[width=0.95\columnwidth]{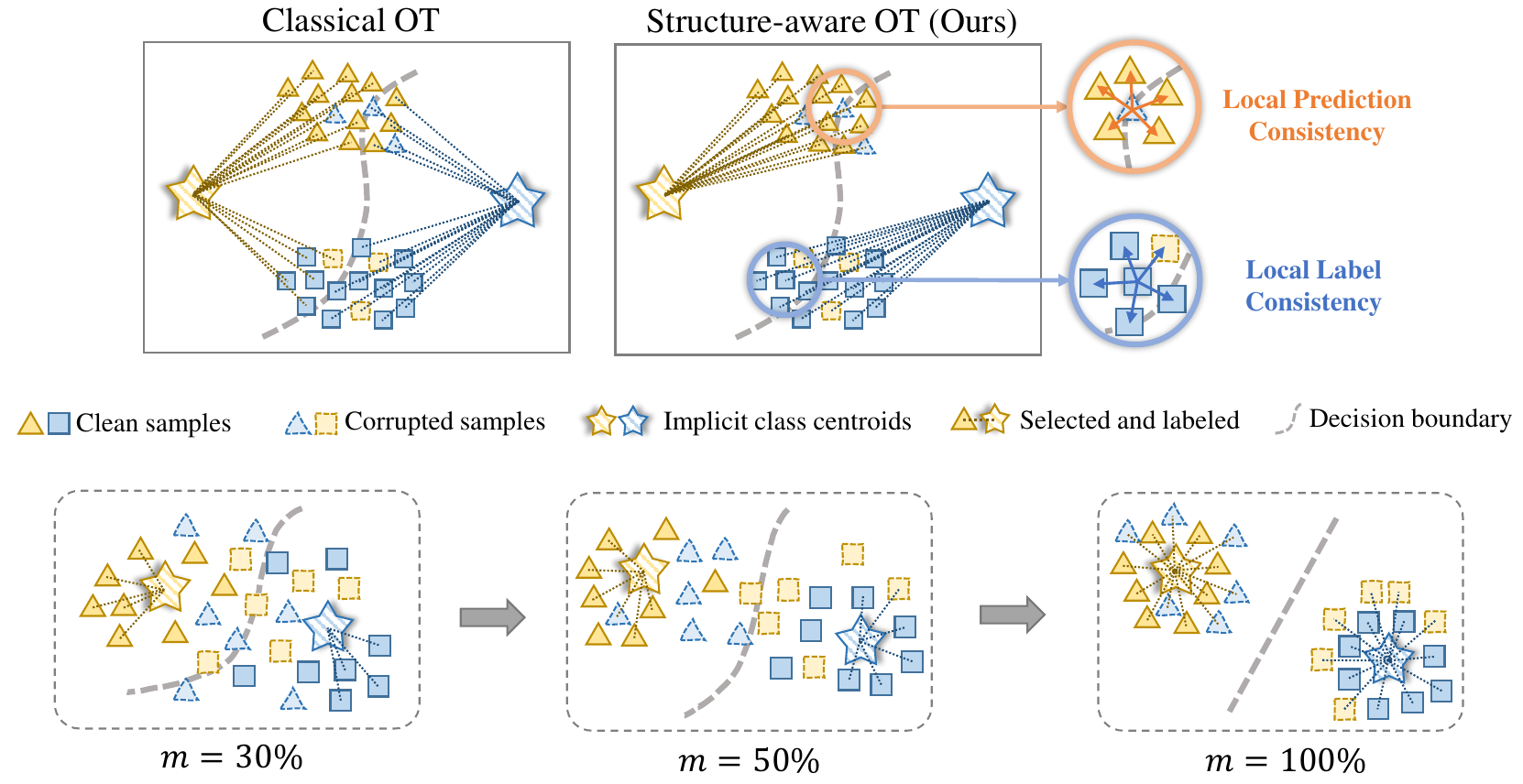}
\end{center}
\vspace{-0.3em}
\caption{\textbf{(Top) Comparison between classical OT and our proposed Structure-aware OT.} 
Classical OT tends to mismatch two nearby samples to two far-away class centroids when the decision boundary is not accurate enough. To mitigate this, our SOT generates local consensus assignments for each sample by preserving prediction-level and label-level consistency. Notably, for vague samples located near the ambiguous decision boundary, SOT rectifies their assignments based on the neighborhood majority consistency.
\textbf{(Bottom) The illustration of our curriculum denoising and relabeling based on proposed CSOT.} The decision boundary refers to the surface that separates two classes by the classifier.
The $m$ represents the curriculum budget that controls the number of selected samples and progressively increases during the training process.}
\label{fig:overview}
\vspace{-1.2em}
\end{figure*}

\section{Methodology}

\paragraph{Problem setup.}
Let $\mathcal{D}_{train}=\{(\mathbf{x}_i,y_i)\}_{i=1}^{N}$ denote the noisy training set, where $\mathbf{x}_i$ is an image with its associated label $y_i$ over $C$ classes, but whether the given label is accurate or not is unknown.
We call the correctly-labeled ones as \textit{clean}, and the mislabeled ones as \textit{corrupted}. 
LNL aims to train a network that is robust to corrupted labels and achieves high accuracy on a clean test set.

\subsection{Structure-Aware Optimal Transport for Denoising and Relabeling}
Even though existing OT-based PL considers the global structure of sample distribution, the intrinsic coherence structure of the samples is ignored. 
Specifically, the cost matrix in OT relies on pairwise metrics and thus two nearby samples could be mapped to two far-away class centroids.
To further consider the intrinsic coherence structure, we propose a Structure-aware Optimal Transport (SOT) for denoising and relabeling, which promotes local consensus assignment by encouraging prediction-level and label-level consistency, as shown in Fig. \ref{fig:overview}.

Our proposed SOT for denoising and relabeling is formulated by adding two local coherent regularized terms based on Problem (\ref{eq:OTPL}). 
Given a cosine similarity $\mathbfS \in \mathbb{R}^{B \times B}$ among samples in feature space, a one-hot label matrix $\mathbfL\in\mathbb{R}^{B \times C}$ transformed from given noisy labels, and a softmax prediction matrix $\mathbfP\in\mathbb{R}_{+}^{B \times C}$,  SOT is formulated as follows
\begin{equation}
    \label{eq:NCOTPL}
    \min_{\mathbfQ\in\mathbfPi(\frac{1}{B}\ones_{B},
        \frac{1}{C}\ones_{C})}
                \left<-\log\mathbfP, \mathbfQ\right>
        +\kappa
            \left(
            \Omega^{\mathbfP}(\mathbfQ)
            + 
            \Omega^{\mathbfL}(\mathbfQ)
            \right),
\end{equation}
where the local coherent regularized terms $\Omega^{\mathbfP}$ and $\Omega^{\mathbfL}$ 
encourages prediction-level and label-level local consistency respectively, and are defined as follows
\begin{equation}
    \Omega^{\mathbfP}(\mathbfQ)
    =-\sum_{i,j}\mathbfS_{ij}
    \sum_{k} \mathbfP_{ik}\mathbfP_{jk}\mathbfQ_{ik}\mathbfQ_{jk}
    =-\left<\mathbfS, 
    \left(\mathbfP\odot\mathbfQ\right)
    \left(\mathbfP\odot\mathbfQ\right)^\top\right>,
\end{equation}
\begin{equation}
    \Omega^{\mathbfL}(\mathbfQ)
    =-\sum_{i,j}\mathbfS_{ij}
    \sum_{k} \mathbfL_{ik}\mathbfL_{jk}\mathbfQ_{ik}\mathbfQ_{jk}
    =-\left<\mathbfS, 
    \left(\mathbfL\odot\mathbfQ\right)
    \left(\mathbfL\odot\mathbfQ\right)^\top\right>,
\end{equation}
where $\odot$ indicates element-wise multiplication.
To be more specific, $\Omega^{\mathbfP}$ encourages assigning larger weight to $\mathbfQ_{ik}$ and $\mathbfQ_{jk}$ if the $i$-th sample is very close to the $j$-th sample, and their predictions $\mathbfP_{ik}$ and $\mathbfP_{jk}$ from the $k$-th class centroid are simultaneously high.
Analogously, $\Omega^{\mathbfL}$ encourages assigning larger weight to those samples whose neighborhood label consistency is rather high.
Unlike the formulation proposed in \cite{structuredOT, chuang2022infoot}, which focuses on sample-to-sample mapping, our method introduces a sample-to-class mapping that leverages the intrinsic coherence structure within the samples.

\subsection{Curriculum and Structure-Aware Optimal Transport for Denoising and Relabeling}

In the early stages of training or in scenarios with a high noise ratio, the predictions and feature representation would be vague and thus lead to the wrong assignments for SOT.
For the purpose of robust clean label identification and corrupted label correction, we further propose a Curriculum and Structure-aware Optimal Transport (CSOT), which constructs a robust curriculum allocator.
This curriculum allocator gradually selects a fraction of the samples with high confidence from the noisy training set, controlled by a budget factor, then assigns reliable pseudo labels for them.

Our proposed CSOT for denoising and relabeling is formulated by introducing new curriculum constraints based on SOT in Problem (\ref{eq:NCOTPL}).
Given curriculum budget factor $m\in[0,1]$, our CSOT seeks optimal coupling matrix $\mathbfQ$ by minimizing following objective
\begin{equation}
\begin{aligned}
\label{eq:CSOT}
&\min_{\mathbfQ}
        \left<-\log\mathbfP, \mathbfQ\right>
        +\kappa
            \left(
            \Omega^{\mathbfP}(\mathbfQ)
            + 
            \Omega^{\mathbfL}(\mathbfQ)
            \right)
\\
\text{s.t.} \quad
\mathbfQ \in 
&\left\{\mathbfQ\in\mathbb{R}_{+}^{B \times C}
        |\mathbfQ\ones_{C}
        \leq
        \frac{1}{B}\ones_{B},
        \mathbfQ^\top\ones_{B}=\frac{
        m
        }{C}\ones_{C}
        \right\}.
\end{aligned}
\end{equation}

Unlike SOT, which enforces an equality constraint on the samples, CSOT relaxes this constraint and defines the total coupling budget as $m \in [0,1]$, where $m$ represents the expected total sum of $\mathbfQ$.
Intuitively speaking, $m=0.5$ indicates that top $50\%$ confident samples are selected from all the classes, avoiding only selecting the same class for all the samples within a mini-batch. 
And the budget $m$ progressively increases during the training process, as shown in Fig. \ref{fig:overview}.

Based on the optimal coupling matrix $\mathbfQ$ solved from Problem (\ref{eq:CSOT}), we can obtain pseudo label by argmax operation, \ie $\hat{y}_i=\argmax_j \mathbfQ_{ij}$.
In addition, we define the general confident scores of samples as $\mathcal{W}=\{w_0,w_1,\cdots, w_{B-1}\}$, where $w_i=\mathbfQ_{i\hat{y}_i}/(m/C)$.
Since our curriculum allocator assigns weight to only a fraction of samples controlled by $m$, we use \texttt{topK}($\mathcal{S}$,$k$) operation (return top-$k$ indices of input set $\mathcal{S}$) to identify selected samples denoted as $\delta_i$
\begin{align}
\delta_i=\left\{
    \begin{aligned}
    1, &\quad  i \in \texttt{topK}(\mathcal{W},
    \left\lfloor mB \right\rfloor )\\
    0, &\quad\text{otherwise}
    \end{aligned}
    \right.
    ,
\end{align}
where $\left\lfloor \cdot \right\rfloor$ indicates the round down operator.
Then the noisy dataset $\mathcal{D}_{train}$ can be splited into $\mathcal{D}_{clean}$ and $\mathcal{D}_{corrupted}$ as follows
\begin{align}
\label{eq:dataset_split}
\begin{aligned}
\mathcal{D}_{clean}\leftarrow
\left\{(\mathbf{x}_i,y_i,w_i)|
\hat{y}_i=y_i, \delta_i=1, \forall(\mathbf{x}_i,y_i)\in \mathcal{D}_{train}
\right\},
\\
\mathcal{D}_{corrupted}\leftarrow
\left\{(\mathbf{x}_i,\hat{y}_i,w_i)|
\hat{y}_i \neq y_i, \forall(\mathbf{x}_i,y_i)\in \mathcal{D}_{train}
\right\}.
\end{aligned}
\end{align}

\subsection{Training Objectives}
To avoid error accumulation in the early stage of training, we adopt a two-stage training scheme.
In the first stage, the model is supervised by progressively selected clean labels and self-supervised by unselected samples.
In the second stage, the model is semi-supervised by all denoised labels.
Notably, we construct our training objective mainly based on Mixup loss $\mathcal{L}^{mix}$ and Label consistency loss $\mathcal{L}^{lab}$ same as NCE \cite{NCE}, and a self-supervised loss $\mathcal{L}^{simsiam}$ proposed in SimSiam \cite{simsiam}.
The detailed formulations of mentioned loss and training process are given in Appendix.
Our two-stage training objective can be constructed as follows
\begin{equation}
\label{eq:L_sup}
    \mathcal{L}^{sup}=
    \mathcal{L}^{mix}_{\mathcal{D}_{clean}} + \mathcal{L}^{lab}_{\mathcal{D}_{clean}} + \lambda_1\mathcal{L}^{simsiam}_{\mathcal{D}_{corrupted}},
\end{equation}
\begin{equation}
\label{eq:L_semi}
    \mathcal{L}^{semi}=
    \mathcal{L}^{mix}_{\mathcal{D}_{clean}} + \mathcal{L}^{lab}_{\mathcal{D}_{clean}} +
    \lambda_2\mathcal{L}^{lab}_{\mathcal{D}_{corrupted}}.
\end{equation}

\section{Lightspeed Computation for CSOT}
The proposed CSOT is a new OT formulation with nonconvex objective function and curriculum constraints, which cannot be solved directly by classical OT solvers. 
To this end, we develop a lightspeed computational method that involves a scaling iteration within a generalized conditional gradient framework to solve CSOT efficiently. 
Specifically, we first introduce an efficient scaling iteration for solving the OT problem with curriculum constraints without considering the local coherent regularized terms, \ie Curriculum OT (COT). 
Then, we extend our approach to solve the proposed CSOT problem, which involves a nonconvex objective function and curriculum constraints.

\subsection{Solving Curriculum Optimal Transport} \label{sec:ER-spOT}

For convenience, we formulate curriculum constraints in Probelm (\ref{eq:CSOT}) in a more general form.
Given two vectors $\bmAlpha$ and $\bmBeta$ that satisfy $\left\|\bmAlpha\right\|_1 \geq \left\|\bmBeta\right\|_1 =m$, a general polytope of curriculum constraints $\cOTpolytope$ is formulated as
\begin{equation}
    \label{Equation:sp_transport_polytope}
    \cOTpolytope=\left\{\mathbfQ\in\mathbb{R}_{+}^{|\bmAlpha| \times |\bmBeta|}
        |\mathbfQ\ones_{|\bmBeta|}\leq\bmAlpha,
        \mathbfQ^\top\ones_{|\bmAlpha|}=\bmBeta
        \right\}.
\end{equation}

For the efficient computation purpose, we consider an entropic regularized version of COT
\begin{equation}
    \label{eq:semi_partialOT_entropic}
        \min_{\mathbfQ\in\cOTpolytope}
            \left<\mathbfC, \mathbfQ\right>
        +\varepsilon\left<\mathbfQ, \log\mathbfQ\right>,
\end{equation}
where we denote the cost matrix $\mathbfC:=-\log\mathbfP$ in Probelm (\ref{eq:CSOT}) for simplicity.
Inspired by \cite{benamou2015iterative}, Problem (\ref{eq:semi_partialOT_entropic}) can be easily re-written as the Kullback-Leibler (KL) projection: 
$\min_{\mathbfQ\in\cOTpolytope}
\varepsilon\text{KL}(\mathbfQ|e^{-\mathbfC/\varepsilon})$.
Besides, the polytope $\cOTpolytope$ can be expressed as an intersection of two convex but not affine sets, \ie
\begin{equation}
\setC_{1}\defeq\left\{\mathbfQ\in\mathbb{R}_{+}^{|\bmAlpha| \times |\bmBeta|}
|\mathbfQ\ones_{|\bmBeta|}\leq\bmAlpha
\right\}
\qandq
\setC_{2}\defeq\left\{\mathbfQ\in\mathbb{R}_{+}^{|\bmAlpha| \times |\bmBeta|}
|\mathbfQ^\top\ones_{|\bmAlpha|}=\bmBeta
\right\}.
\end{equation}

In light of this, Problem (\ref{eq:semi_partialOT_entropic}) can be solved by performing iterative KL projection between $\setC_{1}$ and $\setC_{2}$, namely Dykstra's algorithm \cite{dykstra1983algorithm} shown in Appendix.

\begin{lemma} 
\label{lemma:sinkhorn-like}
(Efficient scaling iteration for Curriculum OT)
When solving Problem (\ref{eq:semi_partialOT_entropic})
by iterating Dykstra's algorithm,
the matrix $\mathbfQ^{(n)}$ at $n$ iteration is a diagonal scaling of $\mathbfK:=e^{-\mathbfC/\varepsilon}$, which is the element-wise exponential matrix of $-\mathbfC/\varepsilon$:
\begin{equation}
\mathbfQ^{(n)}={\rm{diag}}\left(\bmu^{(n)}\right)
                \mathbfK
               {\rm{diag}}\left(\bmv^{(n)}\right),
\end{equation}
where the vectors $\bmu^{(n)}\in \mathbb{R}^{|\bmAlpha|}$, $\bmv^{(n)} \in\mathbb{R}^{|\bmBeta|}$ satisfy $\bmv^{(0)}=\ones_{|\bmBeta|}$ and follow the recursion formula
\begin{equation}
\bmu^{(n)}=\min\left(\frac{\bmAlpha}{\mathbfK\bmv^{(n-1)}},\ones_{|\bmAlpha|}\right)
\quad {\rm{and}} \quad
\bmv^{(n)}=\frac{\bmBeta}{\mathbfK^\top\bmu^{(n)}}.
\end{equation}
\end{lemma}

The proof is given in the Appendix.
Lemma \ref{lemma:sinkhorn-like} allows a fast implementation of Dykstra's algorithm by only performing matrix-vector multiplications.
This scaling iteration for entropic regularized COT is very similar to the widely-used and efficient Sinkhorn Algorithm \cite{sinkhorn}, as shown in Algorithm \ref{algo:fast_dykstra}.
 
\begin{algorithm}[htb]
   \caption{Efficient scaling iteration for entropic regularized Curriculum OT}
   \label{algo:fast_dykstra}
   \begin{algorithmic}[1]
    \STATE {\bfseries Input:} Cost matrix $\mathbfC$, marginal constraints vectors $\bmAlpha$ and $\bmBeta$, entropic regularization weight $\varepsilon$\\
    \STATE Initialize: $\mathbfK \gets e^{-\mathbfC/\varepsilon}$, $\bmv^{(0)} \gets \ones_{|\bmBeta|}$\\
    \STATE Compute: $\mathbfK_{\bmAlpha} \gets \frac{\mathbfK}{\diag(\bmAlpha)\ones_{|\bmAlpha| \times |\bmBeta|}}$, 
    $\mathbfK_{\bmBeta}^\top \gets \frac{\mathbfK^\top}{\diag(\bmBeta)\ones_{|\bmBeta| \times |\bmAlpha|}}$
    \texttt{// Saving computation}\\
	\FOR{$n=1,2,3, \ldots$}
	\STATE $\bmu^{(n)} \gets
            \min \left(\frac{\ones_{|\bmAlpha|}}{\mathbfK_{\bmAlpha} \bmv^{(n-1)}}
            ,\ones_{|\bmAlpha|}\right)$ \label{line:fast_u}\\
        \STATE $\bmv^{(n)} \gets
            \frac{\ones_{|\bmBeta|}}
            {\mathbfK_{\bmBeta}^\top \bmu^{(n)}}$
	\ENDFOR
    \STATE {\bfseries Return:} $\diag(\bmu^{(n)}) \mathbfK \diag(\bmv^{(n)})$
\end{algorithmic}
\end{algorithm}

\subsection{Solving Curriculum and Structure-Aware Optimal Transport}
In the following, we propose to solve CSOT within a Generalized Conditional Gradient (GCG) algorithm \cite{GeneralizedCG} framework, which strongly relies on computing Curriculum OT by scaling iterations in Algorithm \ref{algo:fast_dykstra}.
The conditional gradient algorithm \cite{frank1956algorithm, jaggi2013revisiting}
has been used for some penalized OT problems \cite{ferradans2014regularized, OTDA} or nonconvex Gromov-Wasserstein distances \cite{peyre2016gromov, vayer2020fused, chapel2020partial}, which can be used to solve Problem (\ref{eq:NCOTPL}) directly.

For simplicity, we denote the local coherent regularized terms as $\Omega(\cdot):=\Omega^\mathbfP(\cdot)+\Omega^\mathbfL(\cdot)$, and give an entropic regularized CSOT formulation as follows:
\begin{equation}
    \label{problem:er-spot-art}
        \min_{\mathbfQ\in\cOTpolytope}
            \left<\mathbfC, \mathbfQ\right>
        +\kappa\Omega(\mathbfQ)
        +\varepsilon\left<\mathbfQ, \log\mathbfQ\right>.
\end{equation}
Since the local coherent regularized term $\Omega^\mathbfP(\cdot)$ is differentiable, Problem (\ref{problem:er-spot-art}) can be solved within the GCG algorithm framework, shown in Algorithm \ref{algo:gcg_spOT}.
And the linearization procedure in Line \ref{line:linearization} can be computed efficiently by the scaling iteration proposed in Sec \ref{sec:ER-spOT}.

\begin{algorithm}[htb]
   \caption{Generalized conditional gradient algorithm for entropic regularized CSOT}
   \label{algo:gcg_spOT}
   \begin{algorithmic}[1]
    \STATE {\bfseries Input:} Cost matrix $\mathbfC$, marginal constraints vectors $\bmAlpha$ and $\bmBeta$, entropic regularization weight $\varepsilon$, local coherent regularization weight $\kappa$, local coherent regularization function $\Omega:\mathbb{R}^{|\bmAlpha| \times |\bmBeta|} \to \mathbb{R}$, and its gradient function $\nabla\Omega:\mathbb{R}^{|\bmAlpha| \times |\bmBeta|} \to \mathbb{R}^{|\bmAlpha| \times |\bmBeta|}$\\
    \STATE Initialize: $\mathbfQ^{(0)} \gets \bmAlpha \bmBeta^T$\\
	\FOR{$i=1,2,3, \ldots$}
        \STATE $\mathbfG^{(i)} \gets
            \mathbfQ^{(i)}+
            \kappa\nabla\Omega(\mathbfQ^{(i)})$ 
            \texttt{// Gradient computation}\\
	\STATE $\widetilde{\mathbfQ}^{(i)} \gets       
            \argmin_{\mathbfQ\in\cOTpolytope} 
            \left<\mathbfQ, \mathbfG^{(i)}\right>
            +\varepsilon\left<\mathbfQ, \log\mathbfQ\right>$ 
             \newline\texttt{// Linearization, solved efficiently by Algorithm \ref{algo:fast_dykstra}}  \label{line:linearization}\\
        \STATE Choose $\eta^{(i)}\in [0,1]$ so that it satisfies the Armijo rule
            \texttt{// Backtracking line-search}\\
        \STATE $\mathbfQ^{(i+1)} \gets
            \left(1-\eta^{(i)}\right)\mathbf{Q}^{(i)}
            +\eta^{(i)}\widetilde{\mathbfQ}^{(i)}$
            \texttt{// Update}\\
	\ENDFOR
    \STATE {\bfseries Return:} $\mathbfQ^{(i)}$
\end{algorithmic}
\end{algorithm}

\section{Experiments}\label{Experiments}

\subsection{Implementation Details}
We conduct experiments on three standard LNL benchmark datasets: CIFAR-10 \cite{CIFAR}, CIFAR-100 \cite{CIFAR} and Webvision \cite{webvision}.
We follow most implementation details from the previous work DivideMix \cite{DivideMix} and NCE \cite{NCE}.
Here we provide some specific details of our approach.
The warm-up epochs are set to 10/30/10 for CIFAR-10/100/Webvision respectively.
For CIFAR-10/100, the supervised learning epoch $T_{sup}$ is set to $250$, and the semi-supervised learning epoch $T_{semi}$ is set to $200$.
For Webvision, $T_{sup}=80$ and $T_{semi}=70$.
For all experiments, we set $\lambda_1=1$, $\lambda_2=1$, $\varepsilon=0.1$, $\kappa=1$.
And we adopt a simple linear ramp for curriculum budget, \ie $m=\min(1.0, m_0+\frac{t-1}{T_{sup}-1})$ with an initial budget $m_0=0.3$.
For the GCG algorithm, the number of outer loops is set to 10, and the number for inner scaling iteration is set to 100.
The batch size $B$ for denoising and relabeling is set to $1024$.
More details will be provided in Appendix.

\begin{table}
\centering
\caption{\textbf{Comparison with state-of-the-art methods in test accuracy (\%) on CIFAR-10 and CIFAR-100.} 
The results are mainly copied from \cite{NCE,RRL}.
We present the performance of our CSOT method using the "mean$\pm$variance" format, which is obtained from 3 trials with different seeds.}
\label{tab:cifar}
\resizebox{\linewidth}{!}{ 
\begin{tabular}{r|ccccc|cccc} 
\toprule
Dataset              & \multicolumn{5}{c|}{CIFAR-10}                                                                                              & \multicolumn{4}{c}{CIFAR-100}                                                             \\
Noise type           & \multicolumn{4}{c}{Symmetric}                                                                     & Assymetric             & \multicolumn{4}{c}{Symmetric}                                                             \\
Method/Noise ratio   & 0.2                    & 0.5                    & 0.8                    & 0.9                    & 0.4                    & 0.2           & 0.5                    & 0.8                    & 0.9                     \\ 
\midrule
Cross-Entropy        & 86.8                   & 79.4                   & 62.9                   & 42.7                   & 85.0                   & 62.0          & 46.7                   & 19.9                   & 10.1                    \\
F-correction \cite{F-correction}         & 86.8                   & 79.8                   & 63.3                   & 42.9                   & 87.2                   & 61.5          & 46.6                   & 19.9                   & 10.2                    \\
Co-teaching+ \cite{Co-teaching}         & 89.5                   & 85.7                   & 67.4                   & 47.9                   & -                      & 65.6          & 51.8                   & 27.9                   & 13.7                    \\
PENCIL \cite{PENCIL}               & 92.4                   & 89.1                   & 77.5                   & 58.9                   & 88.5                   & 69.4          & 57.5                   & 31.1                   & 15.3                    \\
DivideMix \cite{DivideMix}            & 96.1                   & 94.6                   & 93.2                   & 76.0                   & 93.4                   & 77.3          & 74.6                   & 60.2                   & 31.5                    \\
ELR \cite{ELR}                 & 95.8                   & 94.8                   & 93.3                   & 78.7                   & 93.0                   & 77.6          & 73.6                   & 60.8                   & 33.4                    \\
NGC \cite{NGC}                 & 95.9                   & 94.5                   & 91.6                   & 80.5                   & 90.6                   & 79.3          & 75.9                   & 62.7                   & 29.8                    \\
RRL \cite{RRL}                 & 96.4                   & 95.3                   & 93.3                   & 77.4                   & 92.6                   & 80.3          & 76.0                   & 61.1                   & 33.1                    \\
MOIT \cite{MOIT}                 & 93.1                   & 90.0                   & 79.0                   & 69.6                   & 92.0                   & 73.0          & 64.6                   & 46.5                   & 36.0                    \\
UniCon \cite{karim2022unicon}              & 96.0                   & 95.6                   & 93.9                   & \textbf{90.8}                   & 94.1                   & 78.9          & 77.6                   & 63.9                   & 44.8                    \\
NCE \cite{NCE}                  & 96.2                   & 95.3                   & 93.9                   & 88.4                   & 94.5                   & \textbf{81.4} & 76.3                   & 64.7                   & 41.1                    \\ 
\midrule
OT Cleaner \cite{OTCleaner}           & 91.4                   & 85.4                   & 56.9                   & -                      & -                      & 67.4          & 58.9                   & 31.2                   & -                       \\
OT-Filter \cite{OT-filter}           & 96.0                   & 95.3                   & 94.0                   & 90.5                   & 95.1                   & 76.7          & 73.8                   & 61.8                   & 42.8                    \\ 
\midrule
\textbf{CSOT (Best)} & \textbf{96.6$\pm$0.10} & \textbf{96.2$\pm$0.11} & \textbf{94.4$\pm$0.16} & 90.7$\pm$0.33 & \textbf{95.5$\pm$0.06} & 80.5$\pm$0.28 & \textbf{77.9$\pm$0.18} & \textbf{67.8$\pm$0.23} & \textbf{50.5$\pm$0.46}  \\
\textbf{CSOT (Last)} & 96.4$\pm$0.18          & 96.0$\pm$0.11          & 94.3$\pm$0.20          & 90.5$\pm$0.36          & 95.2$\pm$0.12          & 80.2$\pm$0.31 & 77.7$\pm$0.14          & 67.6$\pm$0.36          & 50.3$\pm$0.33           \\
\bottomrule
\end{tabular}
} 
\end{table}
  
\begin{table}
    \begin{minipage}{0.48\linewidth}
        \tabcaption{\textbf{Comparison with SOTA methods in top-1 / 5 test accuracy (\%) on the Webvision and ImageNet ILSVRC12 validation sets.} 
        }
        \label{tab:webvision}
        \centering
        \begin{small}
        \centering
        \resizebox{\linewidth}{!}{ 
        \begin{tabular}{r|cc|cc}
        \toprule
        \multicolumn{1}{l|}{}                              & \multicolumn{2}{c|}{Webvision}  & \multicolumn{2}{c}{ILSVRC12}     \\ 
Method                                             & top-1          & top-5          & top-1          & top-5           \\ 
\arrayrulecolor{black}\midrule
F-correction  \cite{F-correction} & 61.12          & 82.68          & 57.36          & 82.36           \\
Decoupling  \cite{Decoupling}     & 62.54          & 84.74          & 58.26          & 82.26           \\
MentorNet  \cite{Mentornet}       & 63.00          & 81.40          & 57.80          & 79.92           \\
Co-teaching  \cite{Co-teaching}   & 63.58          & 85.20          & 61.48          & 84.70           \\
DivideMix \cite{DivideMix}        & 77.32          & 91.64          & 75.20          & 90.84           \\
ELR  \cite{ELR}                   & 76.26          & 91.26          & 68.71          & 87.84           \\
ELR+  \cite{ELR}                  & 77.78          & 91.68          & 70.29          & 89.76           \\
NGC  \cite{NGC}                   & 79.20          & 91.80          & 74.40          & 91.00           \\
RRL  \cite{RRL}                   & 77.80          & 91.30          & 74.40          & 90.90           \\
UniCon \cite{karim2022unicon}                  & 77.60          & 93.44          & 75.29          & 93.72           \\
MOIT \cite{MOIT}                  & 77.90          & 91.90          & 73.80          & 91.70           \\
NCE  \cite{NCE}                   & 79.50          & \textbf{93.80} & 76.30          & \textbf{94.10}  \\ 
\midrule
\textbf{CSOT}                                      & \textbf{79.67} & 91.95          & \textbf{76.64} & 91.67          
        \\ \bottomrule
        \end{tabular}}
        \end{small}
    \end{minipage}%
    \hfill
    \begin{minipage}{0.48\linewidth}
        \tabcaption{\textbf{Time cost (s) for solving CSOT optimization problem of different input sizes.} VDA indicates vanilla Dykstra’s algorithm-based CSOT solver, while ESI indicates the efficient scaling iteration-based solver.}
        \label{tab:CSOT_timecost}
        \centering
        \vspace{1.2em}
        \begin{small}
        \resizebox{\linewidth}{!}{ 
        \begin{tabular}{lcc} 
        \toprule
        $(|\bmAlpha|,|\bmBeta|)$ & VDA-based & ESI-based (Ours)  \\ 
        \hline
        (1024,10)                & 0.83        & \textbf{0.82} \textcolor{red}{$\downarrow$}         \\
        (1024,50)                & 1.00        & \textbf{0.80} \textcolor{red}{$\downarrow$}         \\
        (1024,100)               & 0.87        & \textbf{0.80} \textcolor{red}{$\downarrow$}         \\ 
        \hline
        (50,50)                  & 0.82        & \textbf{0.79} \textcolor{red}{$\downarrow$}         \\
        (100,100)                & 0.88        & \textbf{0.80} \textcolor{red}{$\downarrow$}         \\
        (500,500)                & 0.88        & \textbf{0.87} \textcolor{red}{$\downarrow$}         \\
        (1000,1000)              & 0.94        & \textbf{0.81} \textcolor{red}{$\downarrow$}         \\
        (2000,2000)              & 2.11        & \textbf{0.98} \textcolor{red}{$\downarrow$}         \\
        \textbf{(3000,3000)}              & 3.74        & \textbf{0.99} \textcolor{red}{$\downarrow$}         \\
        \bottomrule
        \end{tabular}}
        \end{small}
    \end{minipage}
    \vspace{-1em}
\end{table}

\subsection{Comparison with the State-of-the-Arts}
\paragraph{Synthetic noisy datasets.}
Our method is validated on two synthetic noisy datasets, \ie CIFAR-10 \cite{CIFAR} and CIFAR-100 \cite{CIFAR}.
Following \cite{DivideMix, NCE}, we conduct experiments with two types of label noise: \textit{symmetric} and \textit{asymmetric}. 
Symmetric noise is injected by randomly selecting a percentage of samples and replacing their labels with random labels. 
Asymmetric noise is designed to mimic the pattern of real-world label errors, \ie labels are only changed to similar classes (\eg cat$\leftrightarrow$dog).
As shown in Tab. \ref{tab:cifar}, our CSOT has surpassed all the state-of-the-art works across most of the noise ratios.
In particular, our CSOT outperforms the previous state-of-the-art method NCE \cite{NCE} by $2.3\%$, $3.1\%$ and $9.4\%$ under a high noise rate of CIFAR-10 sym-0.8, CIFAR-100 sym-0.8/0.9, respectively.

\paragraph{Real-world noisy datasets.}
Additionally, we conduct experiments on a large-scale dataset with real-world noisy labels, \ie WebVision \cite{webvision}.
WebVision contains 2.4 million images crawled from the web using the 1,000 concepts in ImageNet ILSVRC12 \cite{imagenet}.
Following previous works \cite{DivideMix, NCE}, we conduct experiments only using the first 50 classes of the Google image subset for a total of $\sim$61,000 images. 
As shown in Tab. \ref{tab:webvision}, our CSOT surpasses other methods in top-1 accuracy on both Webvision and ILSVRC12 validation sets, demonstrating its superior performance in dealing with real-world noisy datasets.
Even though NCE achieves better top-5 accuracy, it suffers from high time costs (using a single NVIDIA A100 GPU) due to the co-training scheme, as shown in Tab. \ref{tab:time}.

\subsection{Ablation Studies and Analysis}

\paragraph{Effectiveness of CSOT-based denoising and relabeling.}
To verify the effectiveness of each component in our CSOT, we conduct comprehensive ablation experiments, shown in Tab. \ref{tab:ablation}.
Compared to classical OT, Structure-aware OT, and Curriculum OT, our proposed CSOT has achieved superior performance.
Specifically, our proposed local coherent regularized terms $\Omega^\mathbfP$ and $\Omega^\mathbfQ$ indeed contribute to CSOT, as demonstrated in Tab \ref{tab:ablation} (c)(d)(e).
Furthermore, our proposed curriculum constraints yield an improvement of approximately $2\%$ for both classical OT and structure-aware OT, as shown in Tab \ref{tab:ablation} (a)(b)(c).
Particularly, under high noise ratios, the improvement can reach up to $4\%$, which demonstrates the effectiveness of the curriculum relabeling scheme.

\paragraph{Effectiveness of clean labels identification via CSOT.}
As shown in Tab. \ref{tab:ablation} (f), replacing our CSOT-based denoising and relabeling with GMM \cite{DivideMix} for clean label identification significantly degrades the model performance.
This phenomenon can be explained by the clean accuracy during training (Fig. \ref{fig:clean_accuracy}) and clean recall rate (Fig. \ref{fig:recall_rate}), in which our CSOT consistently outperforms other methods in accurately retrieving clean labels, leading to significant performance improvements.
These experiments fully show that our CSOT can maintain both high quantity and high quality of clean labels during training.

\paragraph{Effectiveness of corrupted labels correction via CSOT.}
As shown in Tab. \ref{tab:ablation} (h), only training with identified clean labels leads to inferior model performance.
Furthermore, replacing our CSOT-based denoising and relabeling with confidence thresholding (CT) \cite{fixmatch} for corrupted label correction also degrades the model performance, as shown in Tab. \ref{tab:ablation} (i).
The CT methods assign pseudo labels to samples based on model prediction, which is unreliable in the early training stage, especially under high noise rates.
Our CSOT-based denoising and relabeling fully consider the inter- and intra-distribution structure of samples, yielding more robust labels.
Particularly, our CSOT outperforms NCE and DivideMix significantly in label correction, as demonstrated by the superior corrected accuracy in Fig. \ref{fig:corrected_accuracy} and the improved clarity of the confusion matrix in Fig. \ref{fig:confusion_matrix}.

\begin{table}
\centering
\caption{\textbf{Ablation studies under multiple label noise ratios on CIFAR-10 and CIFAR-100.}
"repl." is an abbreviation for "replaced", and $\mathcal{L}^{ce}$ represents a cross-entropy loss.
GMM refers to the selection of clean labels based on small-loss criterion \cite{DivideMix}. CT (confidence thresholding \cite{fixmatch}) is a relabeling scheme where we set the CT value to 0.95.}
\label{tab:ablation}
\resizebox{\linewidth}{!}{ 
\begin{small}
\begin{tabular}{l|l|cccc|ccc|c} 
\toprule
                                                                                          & Dataset                                                    & \multicolumn{4}{c|}{CIFAR-10}                                      & \multicolumn{3}{c|}{CIFAR-100}                    & \multirow{3}{*}{Avg}                       \\
                                                                                          & Noise type                                                 & \multicolumn{3}{c}{Sym.}                         & Asym.          & \multicolumn{3}{c|}{Sym.}                        &                                            \\
                                                                                          & Method/Noise ratio                                         & 0.5            & 0.8            & 0.9            & 0.4            & 0.5            & 0.8            & 0.9            &                                            \\ 
\midrule
\multirow{5}{*}{\begin{tabular}[c]{@{}l@{}}Denoise\\ Relabeling\\ Technique\end{tabular}} & (a) Classical OT       & 95.45          & 91.95          & 82.35          & 95.04          & 75.96          & 62.46          & 43.28          & 78.07  \\
                                                                                          & (b) Structure-aware OT                                               & 95.86          & 91.87          & 83.29          & 95.06          & 76.20          & 63.73          & 44.57          & 78.65  \\
                                                                                          & (c) CSOT w/o $\Omega^{\mathbfP}$ and $\Omega^{\mathbfL}$             & 95.53          & 93.84          & 89.50          & 95.14          & 75.96          & 66.50          & 47.55          & 80.57  \\
                                                                                          & (d) CSOT w/o $\Omega^{\mathbfP}$                               & 95.77          & 94.08          & 89.97          & 95.35          & 76.09          & 66.79          & 48.13          & 80.88   \\
                                                                                          & (e) CSOT w/o $\Omega^{\mathbfL}$                               & 95.55          & 93.97          & 90.41          & 95.15          & 76.17          & 67.28          & 48.01          & 80.93  \\ 
\midrule
\multirow{5}{*}{\begin{tabular}[c]{@{}l@{}}Learning\\ Technique\end{tabular}}             & (f) GMM + $\mathcal{L}^{sup}$                                  & 92.48          & 80.37          & 31.76          & 90.80          & 69.52          & 48.49          & 20.86          & 62.04   \\
                                                                                          & (g) CSOT repl. $\mathcal{L}^{sup}$ with $\mathcal{L}^{ce}$     & 93.47          & 81.93          & 53.45          & 91.43          & 72.66          & 50.62          & 21.77          & 66.48  \\
                                                                                          & (h) CSOT w/o $\mathcal{L}^{semi}$                              & 95.34          & 93.04          & 88.9           & 94.11          & 75.16          & 61.13          & 36.94          & 77.80  \\
                                                                                          & (i) CSOT repl. correction with
  CT (0.95)                      & 95.46          & 90.73          & 89.09          & 95.21          & 75.85          & 64.28          & 48.76          & 79.91     \\
                                                                                          & (j) CSOT w/o $\mathcal{L}^{simsiam}_{\mathcal{D}_{corrupted}}$ & 95.92          & 94.17          & 89.31          & 95.16          & 76.38          & 66.17          & 45.56          & 80.38  \\ 
\midrule
                                                                                          & CSOT~                                                      & \textbf{96.20} & \textbf{94.39} & \textbf{90.65} & \textbf{95.50} & \textbf{77.94} & \textbf{67.78} & \textbf{50.50} & \textbf{81.85}  \\
\bottomrule
\end{tabular}
\end{small}
} 
\vspace{-1em}
\end{table}

\paragraph{Effectiveness of curriculum training scheme.}
According to the progressive clean and corrupted accuracy during the training process shown in Fig. \ref{fig:clean_accuracy} and Fig. \ref{fig:corrected_accuracy}, our curriculum identification scheme ensures high accuracy in the early training stage, avoiding overfitting to wrong corrected labels.
Note that since our model is trained using only a fraction of clean samples, it is crucial to employ a powerful supervised learning loss to facilitate better learning.
Otherwise, the performance will be poor without the utilization of a powerful supervised training loss, as evidenced in Tab. \ref{tab:ablation} (g).
In addition, the incorporation of self-supervised loss enhances noise-robust representation, particularly in high noise rate scenarios, as demonstrated in our experiments in Tab. \ref{tab:ablation} (j).

\paragraph{Time cost discussion for solving CSOT}
To verify the efficiency of our proposed lightspeed scaling iteration, we conduct some experiments for solving CSOT optimization problem of different input sizes on a single GPU NVIDIA A100.
As demonstrated in Tab. \ref{tab:CSOT_timecost}, our proposed lightspeed computational method that involves an efficient scaling iteration (Algorithm \ref{algo:fast_dykstra}) achieves lower time cost compared to vanilla Dykstra’s algorithm (Algorithm \ref{algo:dykstra}).
Specifically, compared to the vanilla Dykstra-based approach, our efficient scaling iteration version can achieve a speedup of up to 3.7 times, thanks to efficient matrix-vector multiplication instead of matrix-matrix multiplication.
Moreover, even for very large input sizes, the computational time cost does not increase significantly.

\section{Conclusion and Limitation}
In this paper, we proposed Curriculum and Structure-aware Optimal Transport (CSOT), a novel solution to construct robust denoising and relabeling allocator that simultaneously considers the inter- and intra-distribution structure of samples.
Unlike current approaches, which rely solely on the model's predictions, CSOT considers the global and local structure of the sample distribution to construct a robust denoising and relabeling allocator. During the training process, the allocator assigns reliable labels to a fraction of the samples with high confidence, ensuring both global discriminability and local coherence.
To efficiently solve CSOT, we developed a lightspeed computational method that involves a scaling iteration within a generalized conditional gradient framework.
Extensive experiments on three benchmark datasets validate the efficacy of our proposed method. 
While class-imbalance cases are not considered in this paper within the context of LNL, we believe that our approach can be further extended for this purpose.

\begin{figure*}
    \vspace{2em}
    \begin{minipage}{1\linewidth}
      \begin{subfigure}{0.32\linewidth}
        \centering
        \includegraphics[width=\linewidth]{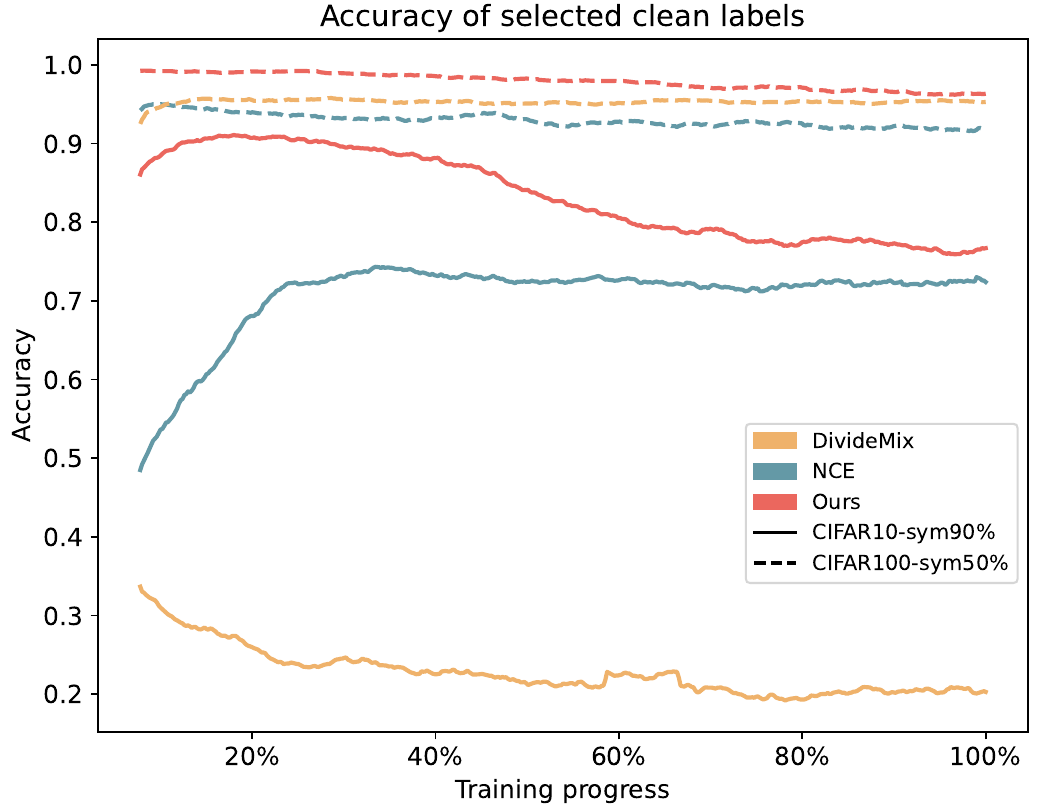}
        \caption{Clean accuracy}
        \label{fig:clean_accuracy}
      \end{subfigure}
      \begin{subfigure}{0.32\linewidth}
        \centering
        \includegraphics[width=\linewidth]{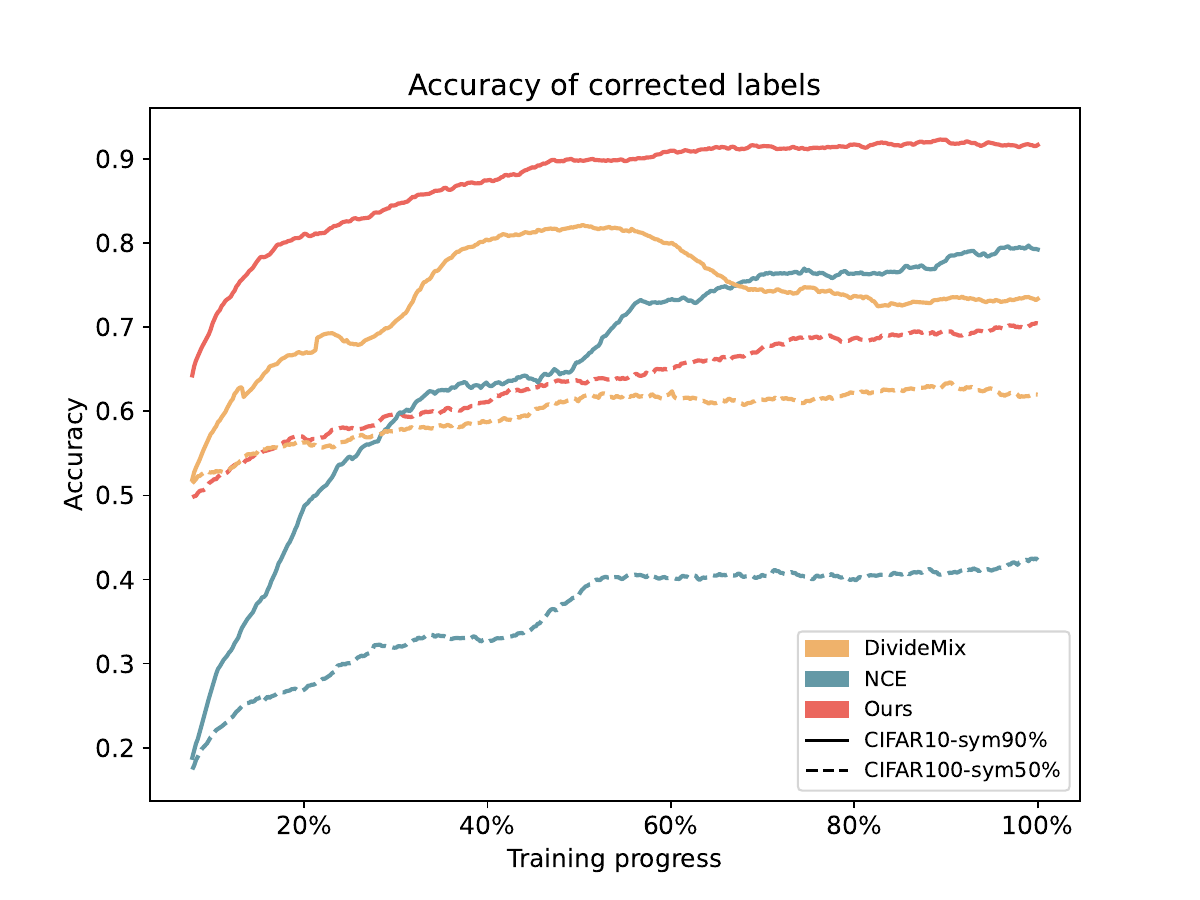}
        \caption{Corrected accuracy}
        \label{fig:corrected_accuracy}
      \end{subfigure}
      \begin{subfigure}{0.32\linewidth}
        \centering
        \includegraphics[width=\linewidth]{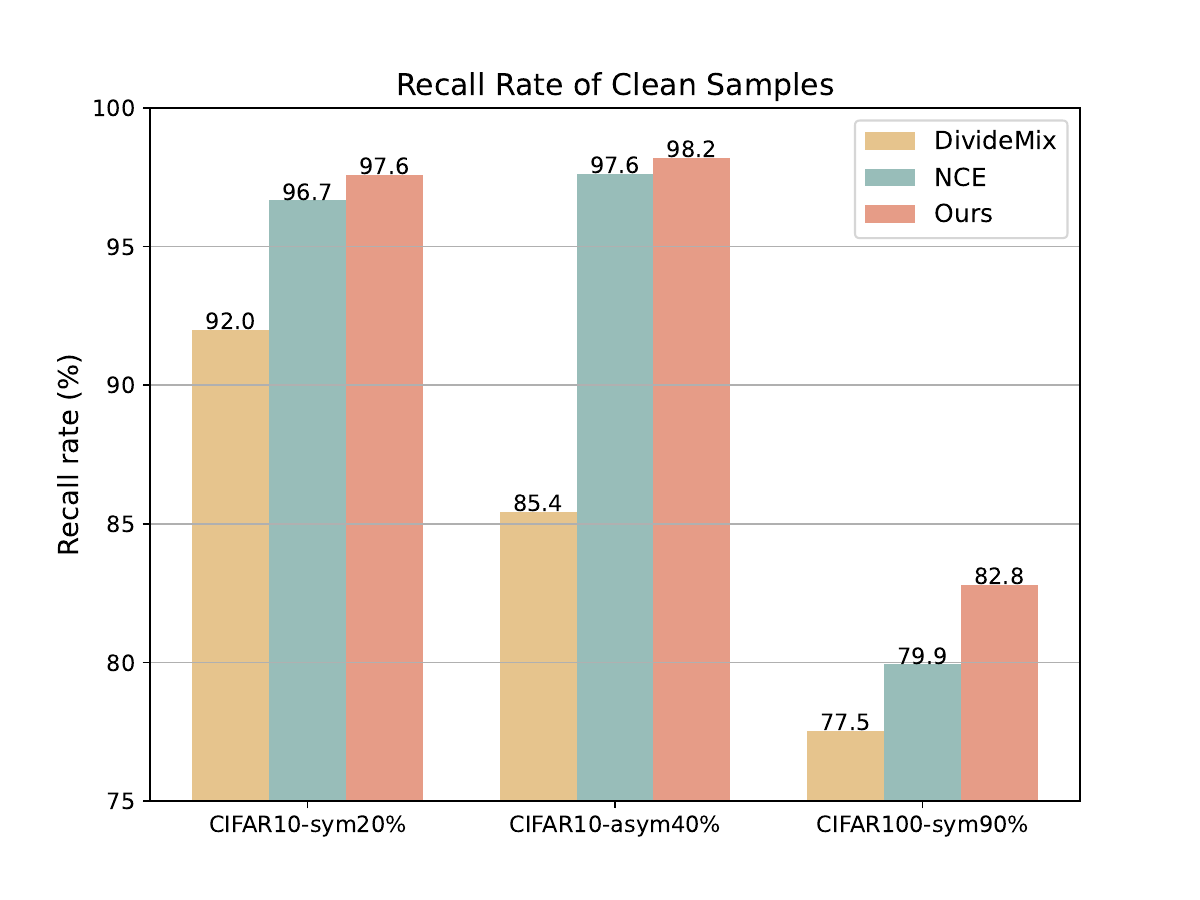}
        \caption{Clean recall rate}
        \label{fig:recall_rate}
      \end{subfigure}
      \figcaption{\textbf{Performance comparison for clean label identification and corrupted label correction.}}
      \label{fig:subfigures}
    \end{minipage}
    \vspace{-1em}
\end{figure*}

\section{Acknowledgement}
This work was supported by NSFC (No.62303319), Shanghai Sailing Program (21YF1429400, 22YF1428800), Shanghai Local College Capacity Building Program (23010503100), Shanghai Frontiers Science Center of Human-centered Artificial Intelligence (ShangHAI), MoE Key Laboratory of Intelligent Perception and Human-Machine Collaboration (ShanghaiTech University), and Shanghai Engineering Research Center of Intelligent Vision and Imaging.

\newpage

{\small
\bibliographystyle{plain}
\bibliography{egbib}
}

\newpage

\appendix

\renewcommand{\thefigure}{S\arabic{figure}}
\renewcommand{\thetable}{S\arabic{table}}
\renewcommand{\theequation}{S\arabic{equation}}
\renewcommand{\thealgorithm}{S\arabic{algorithm}}
\renewcommand{\thelemma}{S\arabic{lemma}}

\section{Supplement for Training Details}
\subsection{Implementation Details}
\paragraph{CIFAR10/100.}
Following previous works \cite{DivideMix, NCE}, we use PreAct ResNet-18 \cite{Resnet-18} as the backbone, and train it using SGD with a momentum of 0.9, a weight decay of 0.0005, and a batch size of 128.
We set the initial learning rate as 0.02, with a cosine learning rate decay schedule.
The hidden layer in SimSiam projection MLP is set to 128-d.

\paragraph{Webvision.}
Following previous works \cite{DivideMix, NCE}, we use inception-resnet v2 \cite{szegedy2017inception} as the backbone, and train it using SGD with a momentum of 0.9, a weight decay of 0.0005, and a batch size of 32.
We set the initial learning rate as 0.01, with a cosine learning rate decay schedule.
The hidden layer in SimSiam projection MLP is set to 384-d.

\paragraph{Other details.}
All experiments are implemented on a single GPU of NVIDIA A100 with 80 GB memory.
We follow DivideMix \cite{DivideMix} and NCE \cite{NCE} to set the hyper-parameters in the mixup loss and label consistency loss.
The loss trade-off weights $\lambda_1$ and $\lambda_2$ are empirically set to $1$, which is similar to NCE \cite{NCE}.
The selection criterion of the hyper-parameters $\varepsilon$ and $\kappa$ in CSOT formulation is analyzed in Sec. \ref{sec:Hyperparameter}.
Our code is modified based on DivideMix \cite{DivideMix} \url{https://github.com/LiJunnan1992/DivideMix} and NCE \cite{NCE} \url{https://github.com/lijichang/LNL-NCE}.
The CSOT solver code is modified based on POT \cite{flamary2021pot}.

\subsection{Training Loss}
To be self-contained, we specify the Mixup loss $\mathcal{L}^{mix}$ and label consistency loss $\mathcal{L}^{lab}$ adopted in NCE \cite{NCE}, and the self-supervised loss $\mathcal{L}^{simsiam}$ proposed in SimSiam \cite{simsiam}.
\paragraph{Mixup loss.}
Mixup \cite{zhang2017mixup} can effectively mitigate noise memorization, and thus mixup regularization can be used to construct augmented samples through linear combinations of existing samples from $\mathcal{D}_{clean}$.
Given two existing samples $(\mathbf{x}_i, y_i)$ and $(\mathbf{x}_j, y_j)$ from $\mathcal{D}_{clean}$, an augmented sample $\widetilde{\mathbf{x}},\widetilde{y}$ can be generated as follows:
\begin{equation}
    \widetilde{\mathbf{x}}=\gamma\mathbf{x}_i+(1-\gamma)\mathbf{x}_j,\quad
    \widetilde{y}=\gamma p_y(y_i) + (1-\gamma)p_y(y_j),
\end{equation}
where $p_y(y_i)$ is the one-hot vector for the given label $y_i$ and $\gamma\sim Beta(\alpha)$ is a mixup ratio and $\alpha$ is a scalar parameter of Beta distribution. The cross-entropy loss applied to $B'$ augmented samples in each training mini-batch is defined as follows:
\begin{equation}
    \mathcal{L}^{mix}=-\frac{1}{B'}\sum_{i=1}^{B'}
    \widetilde{y}_i \log p(y|\widetilde{\mathbf{x}}_i),
\end{equation}
where $p(y|\widetilde{\mathbf{x}}_b)$ is the softmax prediction of a mixup input $\widetilde{\mathbf{x}}_b$.

\paragraph{Label consistency loss.}
Label consistency regularization encourages the fine-tuned model to produce the
same output when there are minor perturbations in the input \cite{fixmatch}.
Hence consistency regularization can be used to further enhance the robustness of the model \cite{englesson2021consistency}.
The label consistency is enforced by minimizing the following loss:
\begin{equation}
   \mathcal{L}^{lab}=-\frac{1}{B'}\sum_{i=1}^{B'} p_y(y_{i})\log p(y|\textbf{Aug}(x_{i})),
\end{equation}
where $\textbf{Aug}(\cdot)$ denotes the function that perturbs the chosen samples using Autoaugment technique proposed in \cite{cubuk2019autoaugment}.

\paragraph{SimSiam loss.}
We simply define a feature extractor as $f$ and a projection layer as $h$.
Given two augmented views $\mathbf{x}_i^1$ and $\mathbf{x}_i^2$ from an image $\mathbf{x}$, we can have $p_i^1=h(f(\mathbf{x}_i^1))$ and $z_i^2=f(\mathbf{x}_i^2)$.
The negative cos similarity is defined as follows:
\begin{equation}
\ell(p_i^1, z_i^2)=-\frac{p_i^1 z_i^2}{\left\|p_i^1\right\|_2 \left\|z_i^2\right\|_2}
\end{equation}
where $\left\|\cdot\right\|_2$ is $\ell_2$-norm. 
To construct the contrastive loss by enforcing the consistency between two positive pairs $(p_i^1, z_i^2)$ and $(p_i^2,z_i^1)$, the SimSiam loss is defined as follows:
\begin{equation}
    \mathcal{L}^{simsiam}=
    -\frac{1}{2B'}\sum_{i=1}^{B'}
    \big(
    \ell(p_i^1, \texttt{stopgrad}(z_i^2))
    + \ell(p_i^2, \texttt{stopgrad}(z_i^1)) \big)
\end{equation}
where $\texttt{stopgrad}(\cdot)$ is a stop-gradient operation that can be easily realized by \texttt{.detach()} operation in PyTorch.

\subsection{Training Process}
\begin{algorithm}[htb]
   \caption{Training process of proposed CSOT}
   \label{algo:training}
   \begin{algorithmic}[1]
    \STATE {\bfseries Input:} Training dataset $\mathcal{D}_{train}$, number of warmup training epochs $T_{warm}$ , number of supervised training epochs $T_{sup}$, number of semi-supervised training epochs $T_{semi}$, initial curriculum budget $m_0$. \\
    \STATE Initialize model parameter $\theta$. \\
    \FOR{$t=1, \ldots, (T_{sup}+T_{semi})$}
        \IF{$t<T_{warm}$}
            \STATE WarmUp($\mathcal{D}_{train}$; $\theta$). \\
        \ELSE
            \STATE Compute the curriculum budget $m=\min(1.0, m_0+\frac{t-1}{T_{sup}-1})$.
            \FOR{$b=1,\ldots,N_{batch}^{relabeling}$} 
                \STATE Draw a mini-batch $\mathcal{X}_b$ from $\mathcal{D}_{train}$.
                \STATE Denoising and relabeling for $\mathcal{X}_b$: solve the Problem (\ref{eq:CSOT}) by Algorithm \ref{algo:gcg_spOT}. \\
            \ENDFOR
            \STATE Use Eq. \ref{eq:dataset_split} to split the training dataset $\mathcal{D}_{train}$ into the clean dataset $\mathcal{D}_{clean}$ and the corrupted dataset $\mathcal{D}_{currupted}$. \\
            \FOR{$b'=1,\ldots,N_{batch}^{train}$} 
                \STATE Draw a mini-batch $\mathcal{X}_{b'}$ from $\mathcal{D}_{clean}$, and draw a mini-batch $\mathcal{U}_{b'}$ from $\mathcal{D}_{currupted}$.
                \IF{$t<T_{sup}$}
                    \STATE $\mathcal{L}=\mathcal{L}^{mix}_{\mathcal{X}_{b'}} + \mathcal{L}^{lab}_{\mathcal{X}_{b'}} + \lambda_1\mathcal{L}^{simsiam}_{\mathcal{U}_{b'}}$.
                \ELSE
                    \STATE $\mathcal{L}=\mathcal{L}^{mix}_{\mathcal{X}_{b'}} + \mathcal{L}^{lab}_{\mathcal{X}_{b'}} + \lambda_2\mathcal{L}^{lab}_{\mathcal{U}_{b'}}$.
                \ENDIF
                \STATE Update model parameter $\theta$ by applying SGD with loss $\mathcal{L}$.
            \ENDFOR
        \ENDIF
    \ENDFOR
    \STATE {\bfseries Return:} Optimal model parameter $\theta$.
\end{algorithmic}
\end{algorithm}
We specify our training process in Algorithm \ref{algo:training}, which mainly includes two parts, \ie denoising and relabeling part, the training part.

\section{Supplement for Experimental Results}

\subsection{Comparison with Prediction- , OT- and SOT-Based Pseudo-Labeling}
\begin{figure*}
\centering
\begin{subfigure}[t]{1\linewidth}
	\centering
	\includegraphics[width=1.\columnwidth]{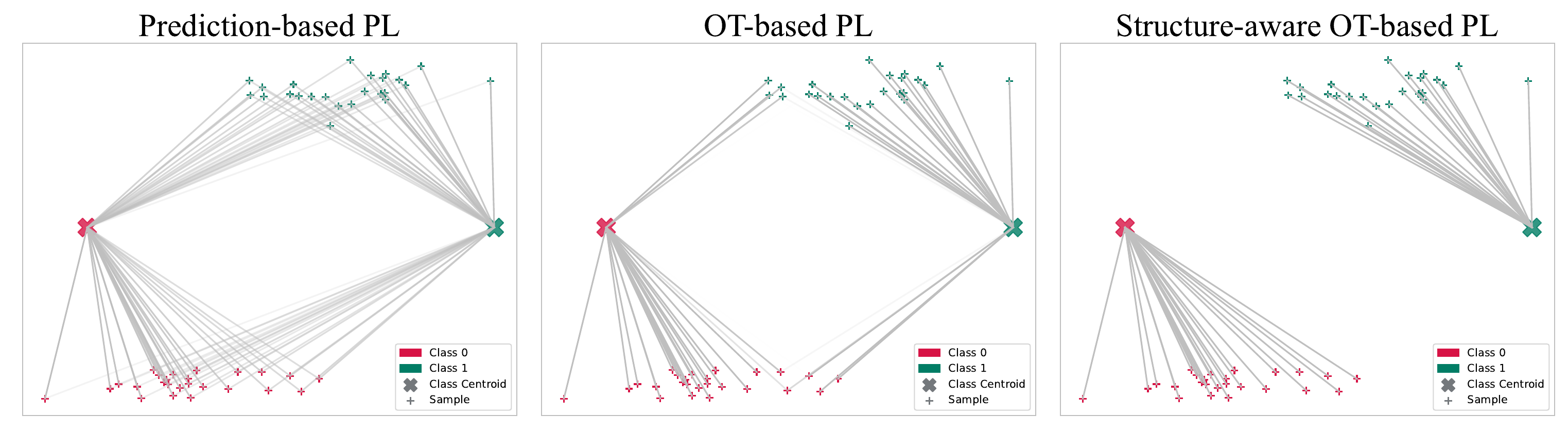}
	\caption{Illustrations of different pseudo-labeling mappings.}
	\label{fig:SOT_mapping}
\end{subfigure}

\begin{subfigure}[t]{1\linewidth}
	\centering
	\includegraphics[width=1.\columnwidth]{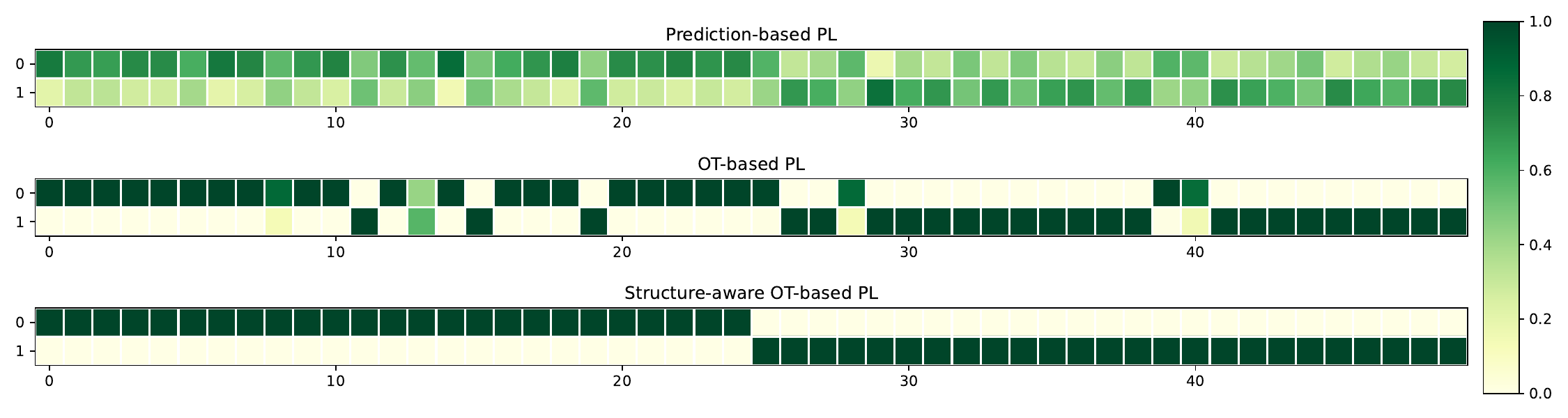}
	\caption{Illustrations of different pseudo-labeling (transposed) coupling matrices.}
	\label{fig:SOT_matrix}
\end{subfigure}

\caption{\textbf{Comparison with prediction- , OT- and SOT-based pseudo-labeling.}
We consider a toy binary classification case for simplicity.
}
\label{fig:SOT_vis}
\end{figure*}
As shown in Fig. \ref{fig:SOT_mapping} and \ref{fig:SOT_matrix}, prediction-based PL generates vague predictions when the class centroids are not discriminative enough.
To explain this, prediction-based PL assigns labels in a per-class manner without considering either the global structure of the sample distribution.
To this end, OT-based PL optimizes the mapping problem considering the inter-distribution matching of samples and classes, and thus produces more discriminative labels.
However, as shown in Fig. \ref{fig:SOT_mapping}, two nearby samples could be mapped to two far-away class centroids, which is not reasonable since it overlooks the inherent coherence structure of the sample distribution, i.e. intra-distribution coherence.
Therefore, \textbf{our proposed SOT encourages generating more robust labels with both global discriminability and local coherence}.

\subsection{Visualization of the Coupling Matrix for CSOT}
\begin{figure*}
\centering
\begin{subfigure}[t]{1\linewidth}
	\centering
	\includegraphics[width=1.\columnwidth]{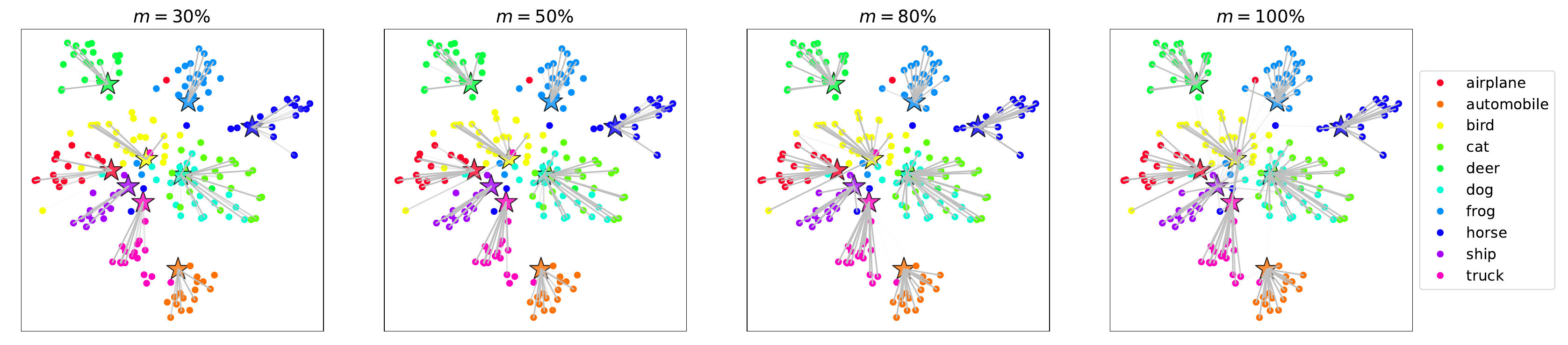}
	\caption{Illustrations of CSOT mappings with different curriculum budgets $m$.}
	\label{fig:CSOT_mapping}
\end{subfigure}

\begin{subfigure}[t]{1\linewidth}
	\centering
	\includegraphics[width=1.\columnwidth]{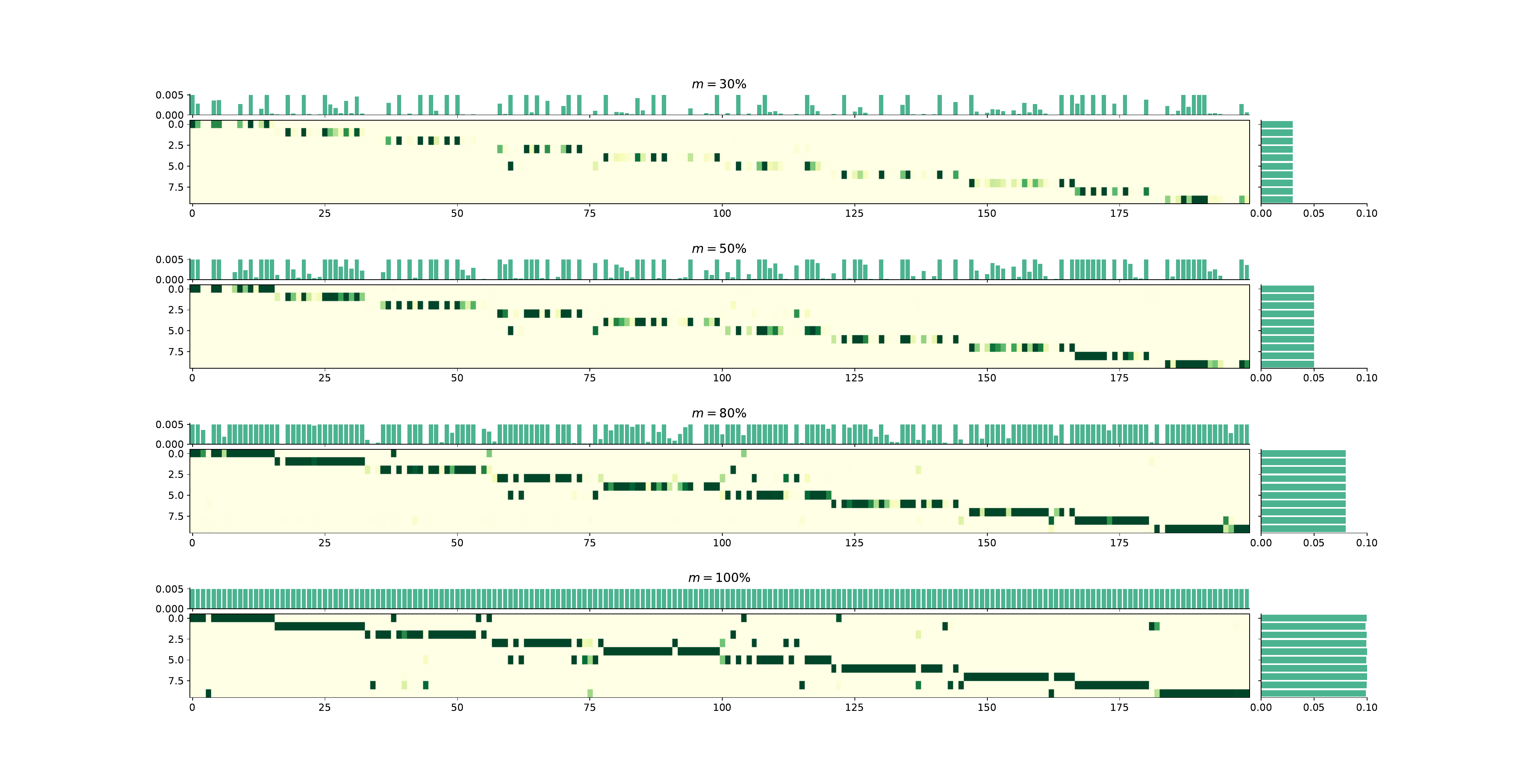}
	\caption{Illustrations of CSOT (transposed) coupling matrices with different curriculum budgets $m$.}
	\label{fig:CSOT_matrix}
\end{subfigure}

\caption{\textbf{Comparison with using different curriculum budgets $m$.}
The samples are plotted as colorful dots and the class centroids are plotted as five-pointed stars, which are colored by their true labels. 
}
\label{fig:CSOT_vis}
\end{figure*}
We visualize randomly selected 200 samples of CIFAR-10 (after 10-epoch warm-up training) and 10 implicit class centroids in feature space in Fig. \ref{fig:CSOT_mapping}.
The feature dots are visualized based on t-SNE \cite{t-SNE}, and the implicit class centroids are obtained by a weighted sum of the softmax prediction scores.
It is evident that the feature space exhibits confusion in the early training stage, particularly among semantically similar classes, such as cat and dog.
Therefore, \textbf{utilizing a full mapping based on OT would lead to incorrect assignments for samples that have not yet been sufficiently learned}. 
Our proposed strategy, on the other hand, demonstrates superiority by selectively assigning reliable labels to a fraction of samples with the highest confidence. 
This approach ensures high training label accuracy and mitigates the negative impact of unreliable labels during the early stages of training.
In addition, we also visualize the coupling matrices, along with their corresponding row and column sum vectors by histograms in Fig. \ref{fig:CSOT_matrix}, which illustrates the partial mapping controlled by curriculum constraints.

\subsection{Convergence of the proposed GCG algorithm for CSOT}
\begin{figure*}
\begin{center}
\includegraphics[width=0.5\columnwidth]{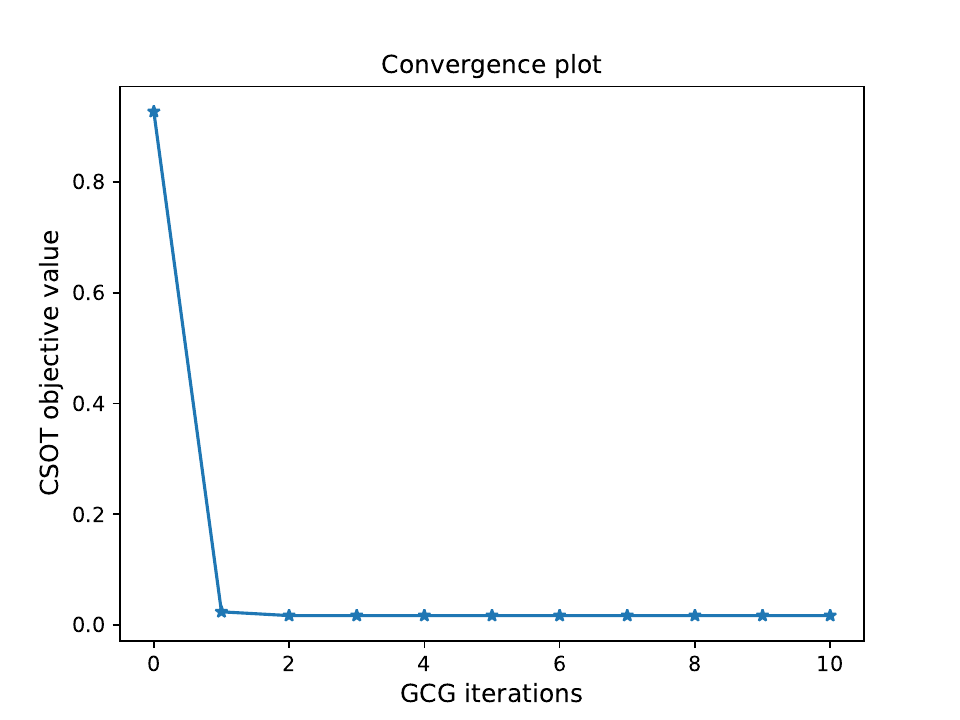}
\end{center}
\caption{\textbf{Convergence behaviour of the generalized conditional gradient (GCG) algorithm for CSOT.} }
\label{fig:convergence}
\end{figure*}
We set the number of outer loops is set to 10, and the number for inner scaling iteration is set to 100. And the curriculum budget $m$ is set to $0.5$, and the local coherent regularized terms weight $\kappa$ is set to $1$.
As demonstrated in Fig. \ref{fig:convergence}, our computational method, which includes a novel scaling iteration within a generalized conditional gradient framework,  \textbf{is capable of optimizing the non-convex objective and converging to a stationary point}.

\subsection{Addictional Results of CSOT}

\paragraph{Effectiveness of CSOT-based denoising and relabeling.}
\begin{figure}[ht]
  \centering
  \begin{subfigure}{0.45\linewidth}
    \centering
    \includegraphics[width=\linewidth]{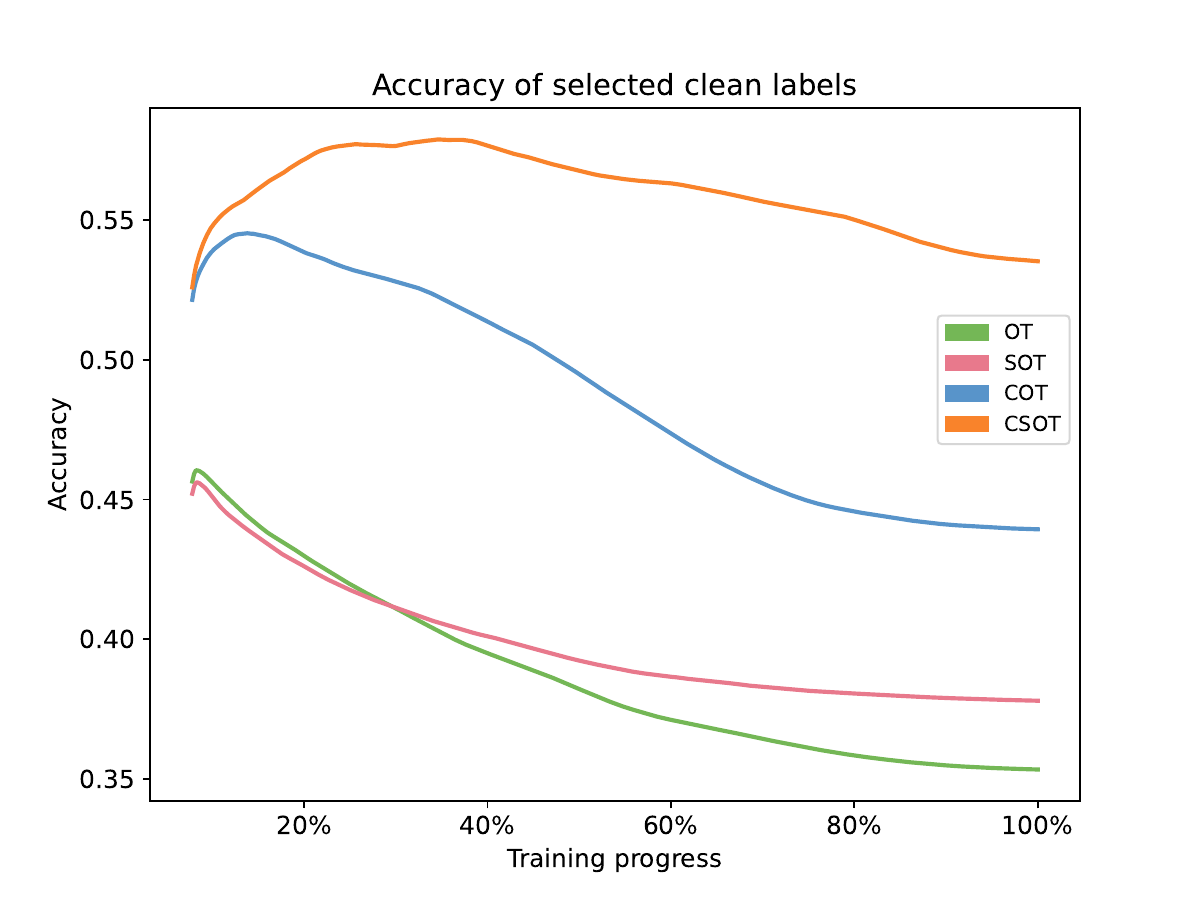}
    \caption{Clean accuracy}
    \label{fig:OT_clean_accuracy}
  \end{subfigure}
  \begin{subfigure}{0.45\linewidth}
    \centering
    \includegraphics[width=\linewidth]{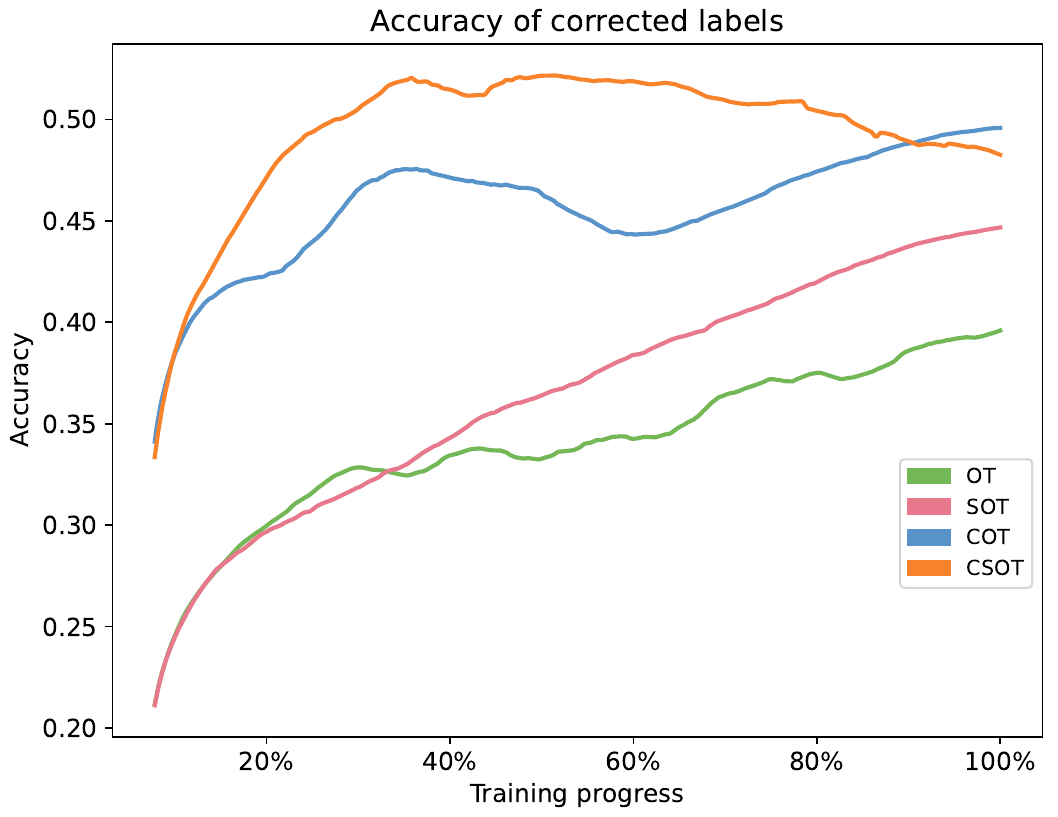}
    \caption{Corrected accuracy}
    \label{fig:OT_corrected_accuracy}
  \end{subfigure}
  \caption{\textbf{Performance comparison for clean label identification and corrupted label correction.}
  The experiments are conducted on CIFAR-100 sym0.9.}
  \label{fig:OT_analysis}
\end{figure}

\begin{figure*}
\centering
\includegraphics[width=0.80\linewidth]{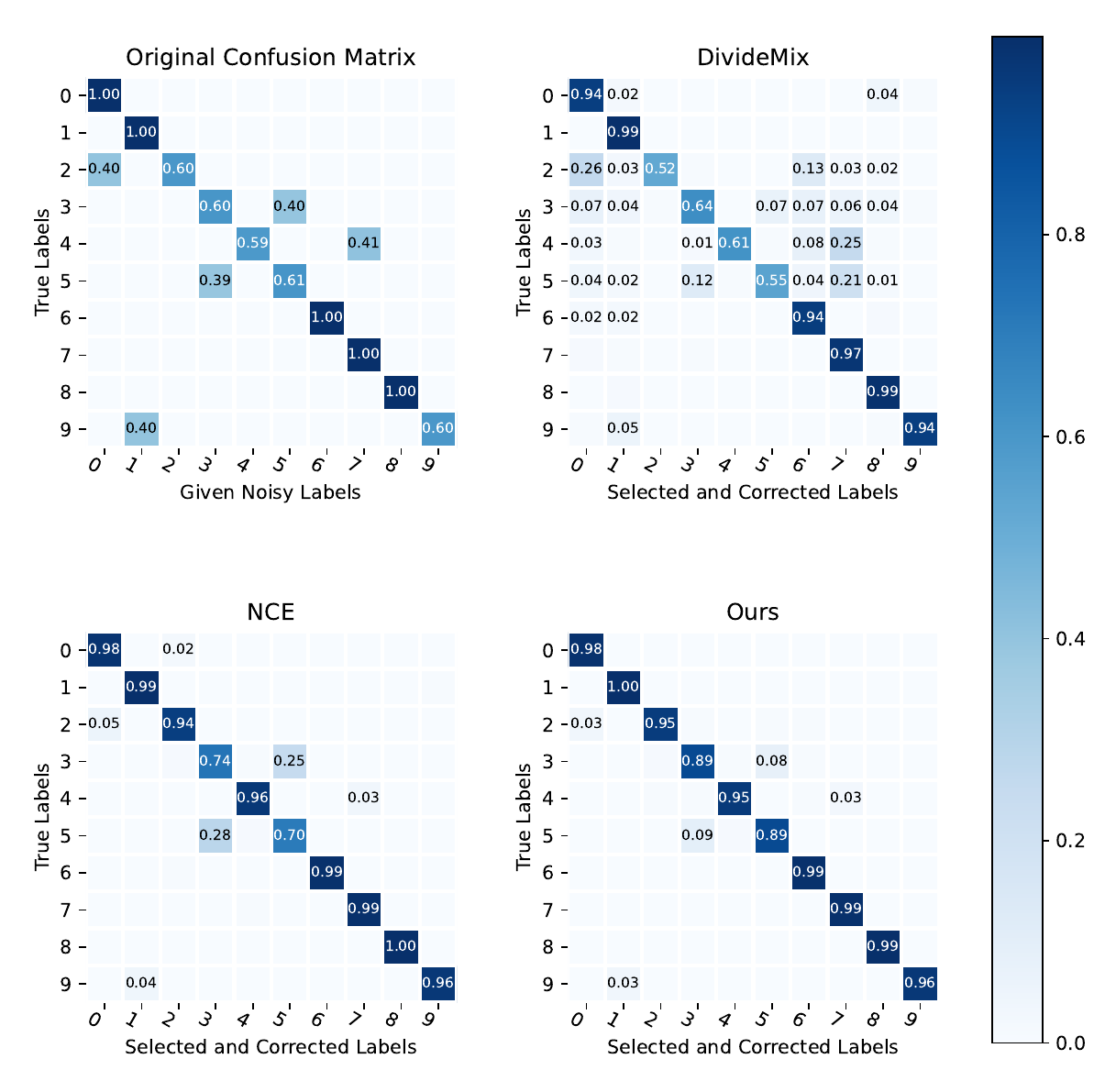}
\figcaption{\textbf{Comparision of confusion matrix on CIFAR-10 assym-40\%.} 
The darker the color on the diagonal elements of the matrix, the higher the accuracy.}\label{fig:confusion_matrix}
\end{figure*}

\begin{table*}[h!t]
\centering
\resizebox{0.3\linewidth}{!}{ 
\begin{tabular}{lc}
    \toprule
    Method                                      & Time cost      \\ 
    \hline
    DivideMix\cite{DivideMix} & 5.1h             \\
    NCE\cite{NCE}             & 6.5h           \\
    \textbf{CSOT}                               & \textbf{4.8h} 
    \\ \bottomrule
    \end{tabular}
} 
\caption{
\textbf{Comparison of total training time (hours) on CIFAR-10.}The experiments are implemented on a single GPU NVIDIA A100.} 
\label{tab:time}
\end{table*}
To further verify the effectiveness of our CSOT for clean label identification and corrupted label correction, we also conduct ablation experiments on OT-, SOT-, COT-, and CSOT-based denoising and relabeling.
As depicted in Fig. \ref{fig:OT_analysis}, \textbf{the incorporation of curriculum constraints ensures high accuracy of clean labels during the early training stage}. This, in turn, facilitates effective learning by providing the model with correct and reliable information, which avoids error accumulation.
Furthermore, \textbf{local coherent regularized terms contribute to improved label correction}.

\paragraph{Result of Clothing1M.}

\begin{table}
\centering
\caption{\textbf{Comparison with state-of-the-art methods in test accuracy (\%) on Clothing1M.} 
}\label{tab:clothing1M}
\resizebox{\linewidth}{!}{ 
\begin{small}
\begin{tabular}{l|cccccc|c}
\toprule
Method   & Meta-L. \cite{li2019learning} & DivideMix \cite{DivideMix} & ELR+ \cite{ELR} & ELR+ \cite{ELR} & RRL \cite{RRL} & NCE\textbf{\ddag} \cite{NCE} & \textbf{CSOT}   \\ 
\midrule
Accuracy & 73.50                                          & 74.48                                       & 72.87                            & 74.80                            & 74.84                           & 74.71                           & \textbf{75.16} 
\\ \bottomrule
\end{tabular}
\end{small}
} 
\end{table}
Clothing1M \cite{clothing1M} is another real-world noisy dataset, which consists of 1 million training images collected from online shopping websites with labels generated from surrounding texts.
We use the augmentation provided in \cite{fixmatch} as $\textbf{Aug}(\cdot)$.
Following the similar setting in NCE \cite{NCE} and DivideMix \cite{DivideMix}, we also \textbf{conduct the experiment on Clothing1M and achieve superior performance} compared to existing approaches, as shown in Tab. \ref{tab:clothing1M}. 
Since NCE utilized an inaccessible data augmentation and hence we reproduce NCE with the augmentation in \cite{fixmatch} for a fair comparison, denoted by \textbf{\ddag}.

\subsection{Hyperparameter Analysis}\label{sec:Hyperparameter}

\paragraph{Entropic regularized weight $\varepsilon$.}
\begin{figure*}
\begin{center}
\includegraphics[width=1.0\columnwidth]{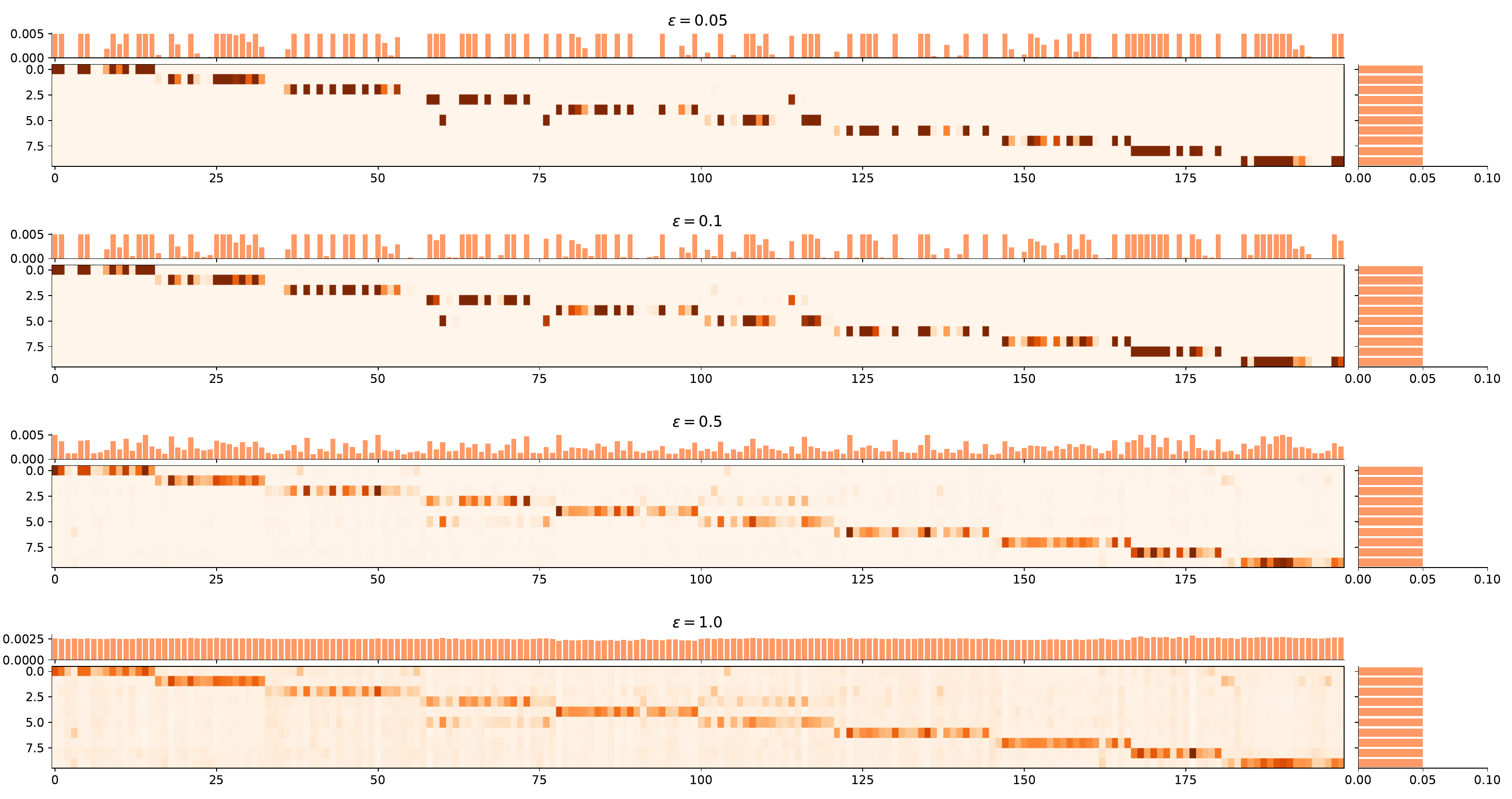}
\end{center}
\caption{\textbf{Visualization of coupling matrix with different entropic regularized weights $\varepsilon$.}
We conduct experiments on randomly selected 200 samples of CIFAR-10 (after 10-epoch warm-up training) and the curriculum budget $m$ is set to $0.5$.}
\label{fig:vareps}
\end{figure*}
When $\varepsilon\to0$, the entropic regularized CSOT formulation becomes closer to the exact CSOT. 
Therefore, in order to obtain a solution that closely approximates the exact CSOT, we prefer to set $\varepsilon$ to a small value.
We visualize the coupling matrix with different $\varepsilon$ in Fig. \ref{fig:vareps}, and it can be observed that the $\varepsilon$ also influences the mapping smoothness.
\textbf{A smaller $\varepsilon$ leads to sharper pseudo labels}. 
To ensure discriminative relabeling and reliable selection, we set $\varepsilon$ to 0.1.

\paragraph{Local coherent regularized weight $\kappa$.}
\begin{figure*}
\begin{center}
\includegraphics[width=0.5\columnwidth]{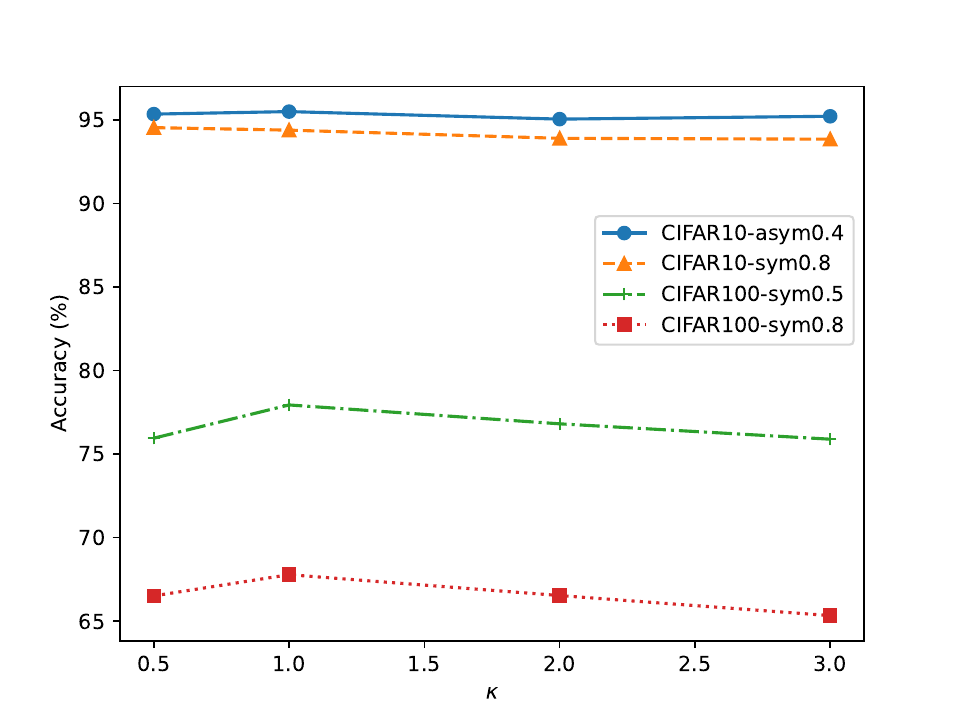}
\end{center}
\caption{\textbf{Sensitivity to the local coherent regularized weight $\kappa$ on four different noisy learning tasks.}}
\label{fig:kappa}
\end{figure*}
The local coherent regularized weight, $\kappa$, determines the strength of local coherent mapping. 
As shown in Fig. \ref{fig:kappa}, we can observe that the performance is not sensitive to the different values of $\kappa$, and it is relatively easy to tune.
It is important to note that \textbf{setting $\kappa$ too high can result in performance degradation, particularly in scenarios with high noise rates}.
This is because the label-level local consistency term $\Omega^{\mathbfL}$, may introduce incorrect consistency in such cases.

\section{More Discussion about CSOT}

\section{Background on Computational Optimal Transport}

\subsection{Sinkhorn's Algorithm}
\begin{algorithm}[htb]
   \caption{Sinhorn algorithm, for entropic regularized classical OT}
   \label{algo:sinkhorn}
   \begin{algorithmic}[1]
    \STATE {\bfseries Input:} Cost matrix $\mathbfC$, marginal constraints vectors $\bmAlpha$ and $\bmBeta$, entropic regularization weight $\varepsilon$\\
    \STATE Initialize: $\mathbfK \gets e^{-\mathbfC/\varepsilon}$, $\bmv^{(0)} \gets \ones_{|\bmBeta|}$\\
    \STATE Compute: $\mathbfK_{\bmAlpha} \gets \frac{\mathbfK}{\diag(\bmAlpha)\ones_{|\bmAlpha| \times |\bmBeta|}}$, 
    $\mathbfK_{\bmBeta}^\top \gets \frac{\mathbfK^\top}{\diag(\bmBeta)\ones_{|\bmBeta| \times |\bmAlpha|}}$
    \texttt{// Saving computation}\\
	\FOR{$n=1,2,3, \ldots$}
	\STATE $\bmu^{(n)} \gets
            \frac{\ones_{|\bmAlpha|}}{\mathbfK_{\bmAlpha} \bmv^{(n-1)}}
            $ \label{line:sinkhorn_u}\\
        \STATE $\bmv^{(n)} \gets
            \frac{\ones_{|\bmBeta|}}
            {\mathbfK_{\bmBeta}^\top \bmu^{(n)}}$
	\ENDFOR
    \STATE {\bfseries Return:} $\diag(\bmu^{(n)}) \mathbfK \diag(\bmv^{(n)})$
\end{algorithmic}
\end{algorithm}

The Sinkhorn algorithm for solving entropic regularized OT problem is presented in Algorithm \ref{algo:sinkhorn}.
It is evident that our proposed scaling iteration is very similar to the existing efficient Sinkhorn algorithm.
The main difference lies in Line \ref{line:fast_u} of Algorithm \ref{algo:sinkhorn}, which corresponds to Line \ref{line:sinkhorn_u} of Algorithm \ref{algo:fast_dykstra}.
Therefore, our scaling iteration shares the same quadratic time complex as the Sinkhorn algorithm.

\subsection{Relation Between Kullback-Leibler Divergence and Entropic Regularized OT}
Given a convex set $\mathcal{C}$, and a matrix $\mathbfM$, the projection according to the Kullback-Leibler (KL) divergence is defined as
\begin{equation}
    \KLprojectionC(\mathbfM)\defeq \argmin_{\mathbfQ\in\setC} \text{KL}(\mathbfQ|\mathbfM).
\end{equation}

According to \cite{benamou2015iterative}, the classical OT can be rewritten in a KL projection form as follows:
\begin{equation}
        \min_{\mathbfQ\in\OTpolytope}
            \left<\mathbfC, \mathbfQ\right>=
        \min_{\mathbfQ\in\OTpolytope}
        \varepsilon\text{KL}(\mathbfQ|e^{-\mathbfC/\varepsilon}),
\end{equation}
which can be interpreted as that solving a classical OT problem is equivalent to solving a KL projection from a given matrix $e^{-\mathbfC/\varepsilon}$ to the constraint $\OTpolytope$.
In light of this, it was proposed in \cite{benamou2015iterative} that when $\mathcal{C}$ is an intersection of closed convex and affine sets, the classical OT problem can be solved by iterative Bregman projections \cite{bregman1967relaxation}. 
However, when $\mathcal{C}$ is an intersection of closed convex but not affine sets, Dykstra's algorithm \cite{dykstra1983algorithm} is employed to guarantee convergence \cite{bauschke2000dykstras}, as iterative Bregman projections do not generally converge to the KL projection on the intersection.

\subsection{Dykstra's Algorithm}
Assume that $\mathcal{C}$ is an intersection of closed convex but not affine sets:
\begin{equation}
        \mathcal{C}=\bigcap_{\ell=1}^{L}  \mathcal{C}_\ell,
\end{equation}
and we extend the indexing of the sets by $L$-periodicity so that they satisfy
\begin{equation}
        \forall n \in \mathbb{N},\quad \mathcal{C}_{n+L}=\mathcal{C}_{n}.
\end{equation}

Dykstra's algorithm \cite{dykstra1983algorithm} starts by initializing 
\begin{equation}
    \mathbfQ^{(0)}=\mathbfK \quad \text{and} \quad \mathbfU^{(0)}=\mathbfU^{(-1)}=\cdots=\mathbfU^{(-L+1)}=\ones.
\end{equation}

One then iteratively defines
\begin{equation}
    \mathbfQ^{(n)}=P^{\text{KL}}_{\setC_n}(\mathbfQ^{(n-1)}\odot\mathbfU^{(n-L)}),
    \quad \text{and} \quad
    \mathbfU^{(n)}=\mathbfU^{(n-L)}\odot\frac{\mathbfQ^{(n-1)}}{\mathbfQ^{(n)}}.
\end{equation}

\subsection{Generalized Conditional Gradient Algorithm}
We are interested in the problem of minimizing under constraints a composite function such as
\begin{equation}
    \min_{\mathbfQ\in\mathcal{C}}=f(\mathbfQ)+g(\mathbfQ),
\end{equation}
where both $f(\cdot)$ is a differentiable and possibly non-convex
function; $g(\cdot)$ is a convex, possibly non-differentiable function; $\mathcal{C}$ denotes any convex and compact set.
One might want to benefit from this composite structure during the optimization procedure. 
For instance, if we have an efficient solver for optimizing
\begin{equation}
    \min_{\mathbfQ\in\mathcal{C}}=\left<\nabla f,\mathbfQ\right>+g(\mathbfQ).
\end{equation}
It is of prime interest to use this solver in the optimization scheme instead of linearizing the whole objective function as one would do with a conditional gradient algorithm \cite{bertsekas1997nonlinear, GeneralizedCG}, as shown in Algorithm \ref{algo:GCG}.

\begin{algorithm}[htb]
   \caption{Generalized conditional gradient algorithm}
   \label{algo:GCG}
   \begin{algorithmic}[1]
    \STATE {\bfseries Input:} A differentiable and possibly non-convex function $f$ and its gradient function $\nabla f$, a convex, possibly non-differentiable function $g$, a convex and compact set $\mathcal{C}$.\\
    \STATE Initialize: $\mathbfQ^{(0)} \in \mathcal{C}$\\
	\FOR{$i=1,2,3, \ldots$}
        \STATE $\mathbfG^{(i)} \gets
            \mathbfQ^{(i)}+
            \nabla f(\mathbfQ^{(i)})$ 
            \texttt{// Gradient computation}\\
	\STATE $\widetilde{\mathbfQ}^{(i)} \gets       
            \argmin_{\mathbfQ\in\mathcal{C}} 
            \left<\mathbfQ, \mathbfG^{(i)}\right>
            +g(\mathbfQ)$ 
             \texttt{// Partial linearization}\\
        \STATE  Find the optimal step $\eta^{(i)}$ with $\Delta\mathbfQ=\widetilde{\mathbfQ}^{(i)}-\mathbfQ^{(i)}$
            \begin{equation*}
                \eta^{(i)}=\argmin_{\eta\in [0,1]} f(\mathbfQ^{(i)}+\eta^{(i)}\Delta\mathbfQ) + g(\mathbfQ^{(i)}+\eta^{(i)}\Delta\mathbfQ)
            \end{equation*}
        or choose $\eta^{(i)}\in [0,1]$ so that it satisfies the Armijo rule.\\
            \texttt{// Exact or backtracking line-search}\\
        \STATE $\mathbfQ^{(i+1)} \gets
            \left(1-\eta^{(i)}\right)\mathbf{Q}^{(i)}
            +\eta^{(i)}\widetilde{\mathbfQ}^{(i)}$
            \texttt{// Update}\\
	\ENDFOR
    \STATE {\bfseries Return:} $\mathbfQ^{(i)}$
\end{algorithmic}
\end{algorithm}

\section{Derivation Details of the Efficient Scaling Iteration Method (Lemma \ref{lemma:sinkhorn-like})}
We have developed a lightspeed computational method that involves a scaling iteration within a generalized conditional gradient framework to solve CSOT efficiently.
Specifically, the efficiency is mainly brought by the scaling iteration method for solving the COT problem (Problem (\ref{eq:semi_partialOT_entropic})), which is proposed in Lemma \ref{lemma:sinkhorn-like}. 

This section presents the derivation details of this efficient scaling iteration method. First, we show that solving COT is equivalent to solving the KL projection problem with the curriculum constraints (Lemma \ref{lemma:KL_projection_whole}).
Then such a KL projection problem can be solved by iterating Dykstra’s algorithm (Lemma \ref{lemma:KL_projection}).
However, Dykstra’s algorithm is based on matrix-matrix multiplication which is computationally extensive.
Therefore, we propose a fast implementation of Dykstra’s algorithm by only performing matrix-vector multiplications, \ie efficient scaling iteration (Lemma \ref{lemma:sinkhorn-like}).

\begin{lemma} 
\label{lemma:KL_projection_whole}
Solving the Problem (\ref{eq:semi_partialOT_entropic}) is equivalent to solving the KL projection problem from the matrix $e^{-\mathbfC/\varepsilon}$ to the curriculum constraint $\cOTpolytope$, \ie
    \begin{equation}
        \min_{\mathbfQ\in\cOTpolytope}
            \left<\mathbfC, \mathbfQ\right>
        +\varepsilon\left<\mathbfQ, \log\mathbfQ\right>
        \Leftrightarrow
        \min_{\mathbfQ\in\cOTpolytope}
        \varepsilon\text{KL}(\mathbfQ|e^{-\mathbfC/\varepsilon}),
    \end{equation}
\end{lemma}

\begin{proof}
\begin{align*}
        &\quad\, \min_{\mathbfQ\in\cOTpolytope}
            \left<\mathbfC, \mathbfQ\right>
        +\varepsilon\left<\mathbfQ, \log\mathbfQ\right>\\ 
        &=\min_{\mathbfQ\in\cOTpolytope}
            \left<\mathbfQ, \mathbfC+\varepsilon\log\mathbfQ\right>\\
        &=\min_{\mathbfQ\in\cOTpolytope}
            \varepsilon\left<\mathbfQ, \mathbfC/\varepsilon+\log\mathbfQ\right>\\
        &=\min_{\mathbfQ\in\cOTpolytope}
            \varepsilon\left<\mathbfQ, \log\frac{\mathbfQ}{e^{-\mathbfC/\varepsilon}}\right>\\
        &=\min_{\mathbfQ\in\cOTpolytope}
            \varepsilon\text{KL}(\mathbfQ|e^{-\mathbfC/\varepsilon})
\end{align*}
\end{proof}

Recall that the curriculum constraints $\cOTpolytope$ can be expressed as an intersection of two convex but not affine sets:
\begin{equation}
\setC_{1}\defeq\left\{\mathbfQ\in\mathbb{R}_{+}^{|\bmAlpha| \times |\bmBeta|}
|\mathbfQ\ones_{|\bmBeta|}\leq\bmAlpha
\right\}
\qandq
\setC_{2}\defeq\left\{\mathbfQ\in\mathbb{R}_{+}^{|\bmAlpha| \times |\bmBeta|}
|\mathbfQ^\top\ones_{|\bmAlpha|}=\bmBeta
\right\}.
\end{equation}

\begin{lemma} 
\label{lemma:KL_projection}
The KL projection from a matrix $\mathbfM$ to the $\setC_{1}$ and $\setC_{2}$ are expressed as
\begin{equation}
    \KLprojectionCI(\mathbfM)={\rm{diag}}\left(\min\left(\frac{\bmAlpha}{{\mathbfM}\ones_{|\bmBeta|}},\ones_{|\bmBeta|}\right)\right){\mathbfM},
\end{equation}
\begin{equation}
    \KLprojectionCII(\mathbfM)={\mathbfM}{\rm{diag}}\left(\frac{\bmBeta}{{\mathbfM}^\top\ones_{|\bmAlpha|}}\right).
\end{equation}
Then the Problem (\ref{eq:semi_partialOT_entropic}) can be solved by Dykstra iterations, presented in Algorithm \ref{algo:dykstra}.
\end{lemma}

Lemma \ref{lemma:KL_projection} can be derived form the Proposition 1 and Proposition 5 in \cite{benamou2015iterative}.

\begin{algorithm}[htb]
   \caption{Dykstra’s algorithm for entropic regularized Curriculum OT}
   \label{algo:dykstra}
   \begin{algorithmic}[1]
    \STATE {\bfseries Input:} Cost matrix $\mathbfC$, marginal constraints vectors $\bmAlpha$ and $\bmBeta$, entropic regularization weight $\varepsilon$\\
    \STATE Initialize: $\mathbfQ^{(0)} \gets e^{-\mathbfC/\varepsilon}$, $\mathbfU'^{(0)} \gets \ones_{|\bmAlpha| \times |\bmBeta|}$, $\mathbfU^{(0)} \gets \ones_{|\bmAlpha| \times |\bmBeta|}$\\
	\FOR{$t=1,2,3, \ldots$}
	\STATE $\mathbfQ'^{(t)} \gets \KLprojectionCI(\mathbfQ^{(t-1)}\odot\mathbfU'^{(t-1)})$ \\
        \STATE $\mathbfU'^{(t)} \gets \mathbfU'^{(t-1)}\odot\frac{\mathbfQ^{(t-1)}}{\mathbfQ'^{(t)}} $\\
        \STATE $\mathbfQ^{(t)} \gets \KLprojectionCII(\mathbfQ'^{(t)}\odot\mathbfU^{(t-1)})$\\
        \STATE $\mathbfU^{(t)} \gets \mathbfU^{(t-1)}\odot\frac{\mathbfQ'^{(t)}}{\mathbfQ^{(t)}} $\\
	\ENDFOR
    \STATE {\bfseries Return:} $\mathbfQ^{(t)}$
\end{algorithmic}
\end{algorithm}

The limitation of Dykstra’s algorithm comes from its computationally extensive matrix-matrix multiplication. To handle this issue, we propose a fast implementation of Dykstra’s algorithm by only performing matrix-vector multiplications, \ie efficient scaling iteration (Lemma \ref{lemma:sinkhorn-like}).

\paragraph{Lemma 1.}
\textit{(Efficient scaling iteration for Curriculum OT)
When solving Problem (\ref{eq:semi_partialOT_entropic})
by iterating Dykstra's algorithm,
the matrix $\mathbfQ^{(n)}$ at $n$ iteration is a diagonal scaling of $\mathbfK:=e^{-\mathbfC/\varepsilon}$, which is the element-wise exponential matrix of $-\mathbfC/\varepsilon$:}
\begin{equation}
\mathbfQ^{(n)}={\rm{diag}}\left(\bmu^{(n)}\right)
                \mathbfK
               {\rm{diag}}\left(\bmv^{(n)}\right),
\end{equation}
\textit{where the vectors $\bmu^{(n)}\in \mathbb{R}^{|\bmAlpha|}$, $\bmv^{(n)} \in\mathbb{R}^{|\bmBeta|}$ satisfy $\bmv^{(0)}=\ones_{|\bmBeta|}$ and follow the recursion formula}
\begin{equation}
\bmu^{(n)}=\min\left(\frac{\bmAlpha}{\mathbfK\bmv^{(n-1)}},\ones_{|\bmAlpha|}\right)
\quad {\rm{and}} \quad
\bmv^{(n)}=\frac{\bmBeta}{\mathbfK^\top\bmu^{(n)}}.
\end{equation}

\begin{proof}
Firstly, let $\bmu^{(1)}:=\min\left(\frac{\bmAlpha}{\mathbfQ^{(0)}\ones_{|\bmBeta|}},\ones_{|\bmAlpha|}\right)$.
Following the Algorithm \ref{algo:dykstra} and Lemma \ref{lemma:KL_projection}, we derive $\mathbfQ'^{(1)}$ and $\mathbfU'^{(1)}$.
Now we have
    \begin{equation*}
    \mathbfQ'^{(1)}=\KLprojectionCI(\mathbfQ^{(0)}\odot\mathbfU'^{(0)})
    =\diag\left(\min\left(\frac{\bmAlpha}{\mathbfQ^{(0)}\ones_{|\bmBeta|}},\ones_{|\bmAlpha|}\right)\right)\mathbfQ^{(0)}
    =\diag\left(\bmu^{(1)}\right)\mathbfQ^{(0)},
    \end{equation*}
    \begin{equation*}
    \mathbfU'^{(1)}=\mathbfU'^{(0)}\odot\frac{\mathbfQ^{(0)}}{\mathbfQ'^{(1)}}
    =\frac{\mathbfQ^{(0)}}{\diag\left(\bmu^{(1)}\right)\mathbfQ^{(0)}}
    =\diag\left(1/\bmu^{(1)}\right) \ones_{|\bmAlpha| \times |\bmBeta|}.
    \end{equation*}

Then let $\bmv^{(1)}:=\frac{\bmBeta}{\mathbfQ^{(0)\top}\bmu^{(1)}}$.
And we derive $\mathbfQ^{(1)}$ and $\mathbfU^{(1)}$ as follows:
    \begin{align*}
    \mathbfQ^{(1)}&= \KLprojectionCII(\mathbfQ'^{(1)}\odot\mathbfU^{(0)})\\
    &=\mathbfQ'^{(1)}\diag\left(\frac{\bmBeta}{\mathbfK^{(1)\top}\ones_{|\bmAlpha|}}\right)\\
    &=\diag\left(\bmu^{(1)}\right) \mathbfQ^{(0)} \diag\left(\frac{\bmBeta}{\mathbfQ^{(0)\top}\diag\left(\bmu^{(1)}\right)\ones_{|\bmAlpha|}}\right)\\
    &=\diag\left(\bmu^{(1)}\right) \mathbfQ^{(0)} \diag\left(\frac{\bmBeta}{\mathbfQ^{(0)\top}\bmu^{(1)}}\right)\\
    &=\diag\left(\bmu^{(1)}\right) \mathbfQ^{(0)} \diag\left(\bmv^{(1)}\right),
    \end{align*}
    \begin{equation*}
    \mathbfU^{(1)}= \mathbfU^{(0)}\odot\frac{\mathbfQ'^{(1)}}{\mathbfQ^{(1)}}
    =\frac{\diag\left(\bmu^{(1)}\right) \mathbfQ^{(0)}}{\diag\left(\bmu^{(1)}\right) \mathbfQ^{(0)} \diag\left(\bmv^{(1)}\right)}
    =\ones_{|\bmAlpha| \times |\bmBeta|} \diag\left(1/\bmv^{(1)}\right).
    \end{equation*}

For simplicity, before deriving $\mathbfQ'^{(2)}$ and $\mathbfU'^{(2)}$, we derive $\mathbfQ^{(1)}\odot\mathbfU'^{(1)}$ firstly:
    \begin{align*}
    \mathbfQ^{(1)}\odot\mathbfU'^{(1)}&= 
    \left(\diag\left(\bmu^{(1)}\right) \mathbfQ^{(0)} \diag\left(\bmv^{(1)}\right)\right)
    \odot
    \left(\diag\left(1/\bmu^{(1)}\right) \ones_{|\bmAlpha| \times |\bmBeta|}\right)\\
    &=\left(\mathbfQ^{(0)} \diag\left(\bmv^{(1)}\right)\right)
    \odot
    \left(\ones_{|\bmAlpha| \times |\bmBeta|}\right)\\
    &=\mathbfQ^{(0)} \diag\left(\bmv^{(1)}\right).
    \end{align*}

Let $\bmu^{(2)}:=\min\left(\frac{\bmAlpha}{\mathbfQ^{(0)} \bmv^{(1)}},\ones_{|\bmAlpha|}\right)$.
We can now derive $\mathbfQ'^{(2)}$ and $\mathbfU'^{(2)}$ as follows:
    \begin{align*}
    \mathbfQ'^{(2)}&=\KLprojectionCI(\mathbfQ^{(1)}\odot\mathbfU'^{(1)})\\
    &=\diag\left(\min\left(\frac{\bmAlpha}{\left(\mathbfQ^{(1)}\odot\mathbfU'^{(1)}\right)\ones_{|\bmBeta|}},\ones_{|\bmAlpha|}\right)\right)
    \left(\mathbfQ^{(1)}\odot\mathbfU'^{(1)}\right)\\
    &=\diag\left(\min\left(\frac{\bmAlpha}{\mathbfQ^{(0)} \diag\left(\bmv^{(1)}\right)\ones_{|\bmBeta|}},\ones_{|\bmAlpha|}\right)\right)
    \mathbfQ^{(0)} \diag\left(\bmv^{(1)}\right)\\
    &=\diag\left(\min\left(\frac{\bmAlpha}{\mathbfQ^{(0)} \bmv^{(1)}},\ones_{|\bmAlpha|}\right)\right)
    \mathbfQ^{(0)} \diag\left(\bmv^{(1)}\right)\\
    &=\diag\left(\bmu^{(2)}\right) \mathbfQ^{(0)} \diag\left(\bmv^{(1)}\right),
    \end{align*}
    \begin{align*}
    \mathbfU'^{(2)}&=\mathbfU'^{(1)}\odot\frac{\mathbfQ^{(1)}}{\mathbfQ'^{(2)}}\\
    &=\left(\diag\left(1/\bmu^{(1)}\right) \ones_{|\bmAlpha| \times |\bmBeta|}\right)
    \odot 
    \frac{\diag\left(\bmu^{(1)}\right) \mathbfQ^{(0)} \diag\left(\bmv^{(1)}\right)}{\diag\left(\bmu^{(2)}\right) \mathbfQ^{(0)} \diag\left(\bmv^{(1)}\right)}\\
    &=\left(\diag\left(1/\bmu^{(1)}\right) \ones_{|\bmAlpha| \times |\bmBeta|}\right)
    \odot 
    \frac{\diag\left(\bmu^{(1)}\right)}{\diag\left(\bmu^{(2)}\right)}\\
    &=\left(\diag\left(1/\bmu^{(1)}\right) \ones_{|\bmAlpha| \times |\bmBeta|}\right)
    \odot 
    \diag\left(\bmu^{(1)}/\bmu^{(2)}\right)\\
    &=\diag\left(1/\bmu^{(2)}\right) \ones_{|\bmAlpha| \times |\bmBeta|}.
    \end{align*}

For simplicity, before deriving $\mathbfQ^{(2)}$ and $\mathbfU^{(2)}$, we derive $\mathbfQ'^{(2)}\odot\mathbfU^{(1)}$ firstly:
    \begin{align*}
    \mathbfQ'^{(2)}\odot\mathbfU^{(1)}&=
    \left(\diag\left(\bmu^{(2)}\right) \mathbfQ^{(0)} \diag\left(\bmv^{(1)}\right)\right)
    \odot
    \left(\ones_{|\bmAlpha| \times |\bmBeta|} \diag\left(1/\bmv^{(1)}\right)\right)\\
    &=\left(\diag\left(\bmu^{(2)}\right) \mathbfQ^{(0)} \right)
    \odot
    \ones_{|\bmAlpha| \times |\bmBeta|}\\
    &=\diag\left(\bmu^{(2)}\right) \mathbfQ^{(0)}.
    \end{align*}

Let $\bmv^{(2)}:=\frac{\bmBeta}{\mathbfQ^{(0)\top}\bmu^{(2)}}$.
We can now derive $\mathbfQ^{(2)}$ and $\mathbfU^{(2)}$ as follows:
    \begin{align*}
    \mathbfQ^{(2)}&= \KLprojectionCII(\mathbfQ'^{(2)}\odot\mathbfU^{(1)})\\
    &=\diag\left(\bmu^{(2)}\right) \mathbfQ^{(0)}
    \diag\left(\frac{\bmBeta}{{(\diag\left(\bmu^{(2)}\right) \mathbfQ^{(0)})}^\top\ones_{|\bmAlpha|}}\right)\\
    &=\diag\left(\bmu^{(2)}\right) \mathbfQ^{(0)}
    \diag\left(\frac{\bmBeta}{\mathbfQ^{(0)\top}\bmu^{(2)}}\right)\\
    &=\diag\left(\bmu^{(2)}\right) \mathbfQ^{(0)} \diag\left(\bmv^{(2)}\right)
    \end{align*}
    \begin{align*}
    \mathbfU^{(2)}&= \mathbfU^{(1)}\odot\frac{\mathbfQ'^{(2)}}{\mathbfQ^{(2)}}\\
    &=\left(\ones_{|\bmAlpha| \times |\bmBeta|} \diag\left(1/\bmv^{(1)}\right)\right)
    \odot
    \frac{\diag\left(\bmu^{(2)}\right) \mathbfQ^{(0)} \diag\left(\bmv^{(1)}\right)}{\diag\left(\bmu^{(2)}\right) \mathbfQ^{(0)} \diag\left(\bmv^{(2)}\right)}\\
    &=\left(\ones_{|\bmAlpha| \times |\bmBeta|} \diag\left(1/\bmv^{(1)}\right)\right)
    \odot
    \frac{\diag\left(\bmv^{(1)}\right)}{\diag\left(\bmv^{(2)}\right)}\\
    &=\left(\ones_{|\bmAlpha| \times |\bmBeta|} \diag\left(1/\bmv^{(1)}\right)\right)
    \odot
    \diag\left(\bmv^{(1)}/\bmv^{(2)}\right)\\
    &=\ones_{|\bmAlpha| \times |\bmBeta|} \diag\left(1/\bmv^{(2)}\right)
\end{align*}

To conclude, it can be easily summarized that
\begin{equation*}
\mathbfQ^{(n)}={\rm{diag}}\left(\bmu^{(n)}\right)
                \mathbfK
               {\rm{diag}}\left(\bmv^{(n)}\right),
\end{equation*}
where $\bmu^{(n)}=\min\left(\frac{\bmAlpha}{\mathbfK\bmv^{(n-1)}},\ones_{|\bmAlpha|}\right)$, $\bmv^{(n)}=\frac{\bmBeta}{\mathbfK^\top\bmu^{(n)}}$, and $\bmv^{(0)}=\ones_{|\bmBeta|}$.

\end{proof}

\newpage

\end{document}